\documentclass[11pt]{article}

\usepackage[final]{acl}

\usepackage{times}
\usepackage{latexsym}

\usepackage[T1]{fontenc}

\usepackage[utf8]{inputenc}

\usepackage{microtype}

\usepackage{inconsolata}

\usepackage{amsmath, amssymb, amsthm}
\usepackage{amsthm}
\usepackage{mathtools}
\usepackage{xurl}
\usepackage{tabularx}
\newtheorem{theorem}{Theorem}[section]   
\newtheorem{lemma}[theorem]{Lemma}       

\theoremstyle{definition}

\theoremstyle{remark}
\usepackage{algorithm}
\usepackage{algorithmicx}
\usepackage{algpseudocode}

\usepackage[table]{xcolor}
\usepackage{booktabs}
\usepackage{graphicx}
\usepackage{caption}
\usepackage{hyperref}
\usepackage[all]{hypcap}
\title{BranPO: Scalable Contrastive Branch Sampling for Long-Horizon Agentic Reinforcement Learning}
\newcommand{\think}[1]{\textcolor{blue}{\texttt{<thinking>}} #1 \textcolor{blue}{\texttt{</thinking>}}}
\newcommand{\search}[1]{\textcolor{cyan}{\texttt{<search>}} #1 \textcolor{cyan}{\texttt{</search>}}}
\newcommand{\info}[1]{\textcolor{brown}{\texttt{<information>}} #1 \textcolor{brown}{\texttt{</information>}}}
\newcommand{\answer}[1]{\textcolor{purple}{\texttt{<answer>}} #1 \textcolor{purple}{\texttt{</answer>}}}

\author{
 \textbf{Yubao Zhao\textsuperscript{1}},
 \textbf{Weiquan Huang\textsuperscript{1}},
 \textbf{Sudong Wang\textsuperscript{1}},
 \textbf{Ruochen Zhao\textsuperscript{2}},
\\
 \textbf{Chen Chen\textsuperscript{2}},
 \textbf{Yao Shu\textsuperscript{1}},
 \textbf{Chengwei Qin\textsuperscript{1}}
\\
\\
 \textsuperscript{1}The Hong Kong University of Science and Technology (Guangzhou),\\
 \textsuperscript{2}Nanyang Technological University
\\
 \small{
   \textbf{Correspondence:} \href{mailto:yzhao440@connect.hkust-gz.edu.cn}{yzhao440@connect.hkust-gz.edu.cn}, \href{mailto:chengweiqin@hkust-gz.edu.cn}{chengweiqin@hkust-gz.edu.cn}
 }
}

\begin{document}
\maketitle

\begin{abstract}
Agentic reinforcement learning enables large language models to perform multi-turn planning and tool use, but long-horizon training remains challenging under sparse trajectory-level rewards, where a single outcome is uniformly assigned to all decisions. Prior methods introduce finer-grained supervision via tree-based exploration or process-level evaluation, but often incur high cost or produce noisy credit signals. In agentic trajectories, early mistakes may still be corrected by later actions, while seemingly promising intermediate states can fail due to poor subsequent decisions. We call this property non-monotonic correctness, which makes outcome rewards or state values insufficient for guiding what actions should be taken from each state. To address this, we propose Branching Relative Policy Optimization (\textbf{BranPO}), a value-free method that constructs localized contrastive supervision without dense rewards. BranPO truncates trajectories at intermediate prefixes and resamples continuations to form contrastive branches that share the same prefix but diverge in final outcomes, thereby isolating decisions that drive success or failure. We further introduce difficulty-aware branch sampling and Redundant Step Masking to improve sampling efficiency and suppress redundant updates. Experiments show that BranPO consistently outperforms diverse baseline categories across multiple multi-hop QA benchmarks without additional training cost, and generalizes to broader long-horizon agentic tasks with consistent improvements. Our code is available at https://github.com/YubaoZhao/BranPO.
\end{abstract}

\section{Introduction}
Large language models (LLMs) have increasingly evolved from passive text generators into autonomous agents capable of multi-turn interaction, planning, and tool use. Agentic reinforcement learning (RL) further improves these capabilities through policy optimization \cite{liu2025deepseek, minimaxm2, kimi-researcher}, enabling progress on long-horizon tasks such as deep research \cite{team2025tongyi}, code generation \cite{feng2025towards}, and GUI-based assistance \cite{qin2025ui}

\begin{figure}
    \centering
    \includegraphics[width=\linewidth]{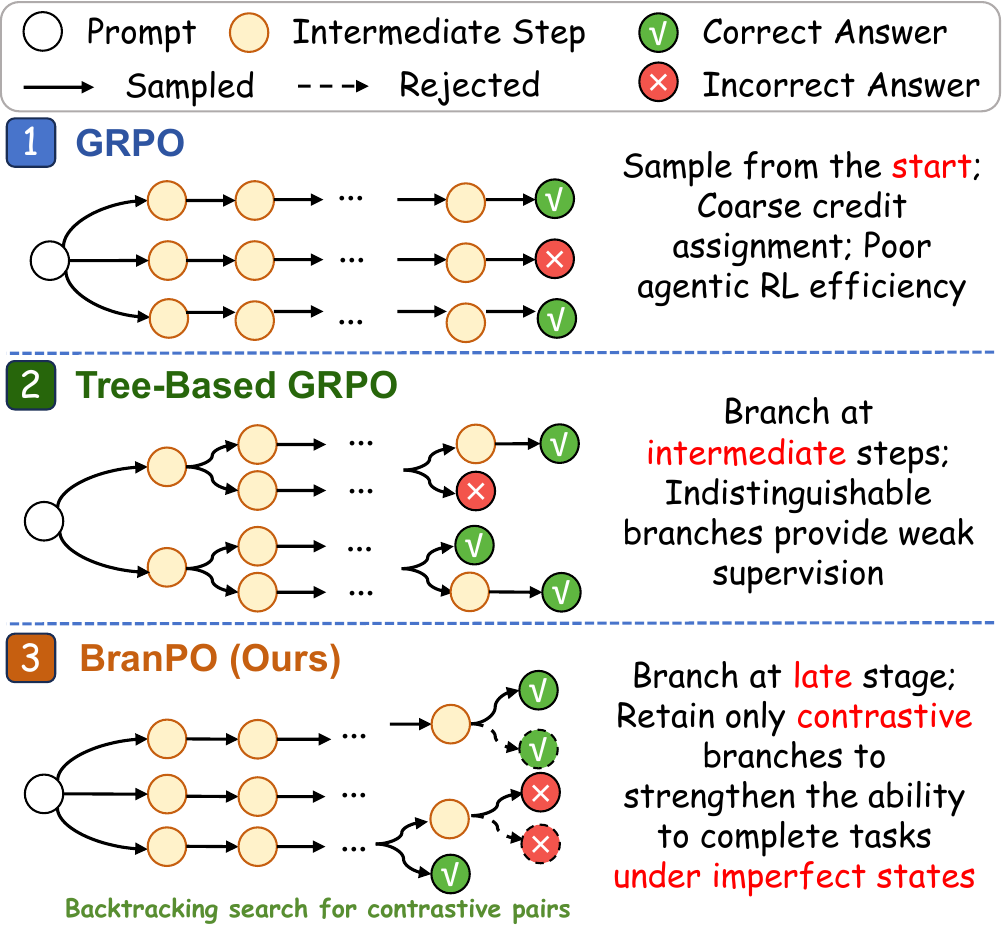}
    \caption{Comparison of sampling process between GRPO, tree-based GRPO, and BranPO.}
    \label{fig:method_compare}
\end{figure}
Despite recent progress, training multi-turn agents remains difficult under sparse rewards \cite{xue2025simpletir, qian2025toolrl, song2025r1}, where a single trajectory-level outcome supervises all actions \cite{liu2025gem}. Existing methods seek finer-grained credit assignment through process reward models \cite{zhang2025process, zhang2025criticsearch, zeng2025reinforcing}, tree-based \cite{yang2025treerpo,ji2025tree}, or entropy-guided branch sampling~\cite{shen2025carl,dong2025aepo,dong2025arpo}. However, these approaches remain limited in long-horizon interaction. Process rewards often rely on LLM-as-a-Judge or task-specific evaluators, which can be unreliable under incomplete or delayed evidence. Tree-based and entropy-guided methods require additional branch rollouts; yet as horizons grow, outcome-critical steps become sparse, so many sampled branches have continuations that either all succeed or all fail, incurring substantial computation for weak learning signals.

We argue that this limitation reflects a structural property of multi-turn agentic tasks, which we call \textbf{non-monotonic correctness}. In such tasks, correctness does not accumulate monotonically along a trajectory. A flawed intermediate state can be recovered through later corrective actions, whereas a promising prefix may still fail due to poor subsequent decisions. This makes isolated judgments of intermediate states unreliable, because the value of a prefix depends on how the agent continues from it. Earlier uncertainty or errors may still be resolved through later query reformulation, evidence integration, or final answer construction, making late-stage actions especially consequential in determining whether a prefix is corrected into success or carried into failure. Therefore, effective supervision should prioritize continuations from intermediate states, including imperfect ones, while reinforcing high-quality earlier decisions, so that the agent learns to build useful prefixes and correct rather than propagate earlier uncertainty.

Motivated by this, we propose \textbf{Bran}ching Relative \textbf{P}olicy \textbf{O}ptimization (\textbf{BranPO}), a contrastive dynamic branch sampling method for multi-turn agent training. As shown in Figure~\ref{fig:method_compare}, BranPO turns completed trajectories into targeted training signals by searching backward for late-stage branching points, avoiding expensive tree expansion over the full interaction horizon. From each selected prefix, BranPO resamples alternative continuations and retains only contrastive branches whose continuations lead to different final outcomes, concentrating learning on decisions that truly affect task success. BranPO further applies hierarchical group-relative optimization: prefix-level advantages distinguish more and less promising branching contexts, while branch-level advantages contrast successful and failed continuations under the same shared prefix. This yields a compact, high-signal objective for long-horizon, non-monotonic agentic tasks.

To improve efficiency, we introduce difficulty-aware branch sampling, which adaptively allocates sampling budgets based on accuracy. To mitigate the continuation bias, we propose Redundant Step Masking (RSM), which identifies and masks redundant tail steps during policy updates, thereby reinforcing informative supervision.

Across diverse multi-hop QA benchmarks, BranPO achieves consistent improvements over GRPO and strong agentic RL variants, including explicit process-reward methods, tree-based branch sampling, and entropy-guided branch sampling, without increasing the training budget. Additionally, BranPO acquires long-horizon search behaviors substantially faster, and its benefits generalize to 30+ turn agentic tasks and tasks beyond search.

In summary, our contributions are as follows:
\begin{itemize}

\item We introduce BranPO, the first agentic RL algorithm that adapts policy optimization to the non-monotonic credit structure of long-horizon multi-turn agentic tasks via contrastive, outcome-divergent branch learning.
\item We introduce difficulty-aware branch sampling and Redundant Step Masking to improve exploration efficiency and suppress redundant behaviors.
\item Extensive experiments demonstrate that BranPO consistently outperforms prior methods in long-horizon agentic tasks.
\end{itemize}

\begin{figure*}[t]
  \centering
  \includegraphics[width=\textwidth]{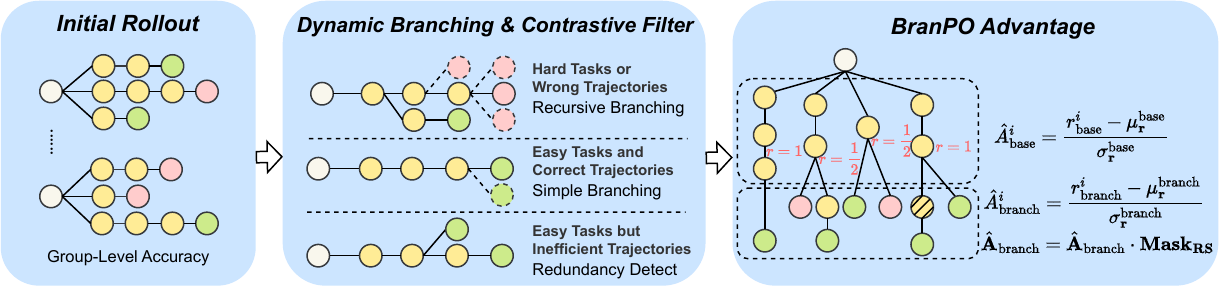}
  \caption{Overview of BranPO. After the initial rollout, group accuracy is computed and branching budgets are assigned based on task accuracy and trajectory reward. Simple branching is applied to correct trajectories in easy tasks, while recursive branching is used for hard tasks or incorrect trajectories. Redundancy checks are performed on inefficient trajectories within easy tasks. Within a task, shared prefixes from all trajectories form a base group, while branched continuations from all trajectories constitute a branch group for relative advantage computation, with advantages corresponding to inefficient steps masked.}
  \label{fig:method}
  \vskip -0.1in
\end{figure*}

\section{Related Work}
\subsection{Search Agents}
Search agents have emerged as a pivotal approach for open-domain problem solving by integrating reasoning with external tool use. Representative studies such as Search-o1 \cite{li2025search}, WebThinker \cite{li2025webthinker}, and Search-R1 \cite{jin2025search} demonstrate how iterative reason-retrieve-synthesize loops can be optimized through reinforcement learning. Despite recent progress in query generation \cite{jiang2025deepretrieval}, self-reflection \cite{fu2025re}, and data efficiency \cite{jiang2025s3}, search agents continue to face significant challenges in long-horizon tasks \cite{gao2025beyond}. In particular, the lack of fine-grained process supervision limits existing feedback mechanisms from effectively guiding intermediate steps within long interaction sequences \cite{liu2026pica}.

\subsection{Credit Assignment in RL}
Credit assignment remains a fundamental bottleneck in training LLMs for complex, long-horizon tasks. Standard RL approaches \cite{ouyang2022training} typically rely on sparse, outcome rewards, resulting in high variance and ambiguity. Existing methods for finer-grained supervision fall into two categories: (1) \emph{Monte Carlo estimation}, which branches from intermediate states and re-runs rollouts to estimate expected rewards. For example, VinePPO \cite{kazemnejad2024vineppo} operates at the token level, Tree-GRPO \cite{ji2025tree} and ReasonRAG \cite{zhang2025process} at the turn level, while other work exploits uncertainty signals such as entropy \cite{shen2025carl,dong2025aepo,dong2025arpo} to guide efficient branching. (2) \emph{Explicit step-level reward methods}, which directly construct step-level supervision, such as GiGPO \cite{feng2025groupingroup}, StepSearch \cite{wang2025stepsearch}, MT-GRPO \cite{zeng2025reinforcing}, and CriticSearch \cite{zhang2025criticsearch}. However, explicit step-level annotation scales poorly. Fine-grained credit assignment based on outcome-level rewards remains a more scalable and promising direction for long-horizon agentic learning.

\section{Methodology}

\subsection{Non-Monotonic Correctness}
\label{sec:observation}

We formulate multi-turn agentic reasoning as a finite-horizon Markov Decision Process where the model interacts with an environment over multiple steps. At each turn $t$, the agent observes the full interaction history $h_t = (a_1, o_1, \dots, a_{t-1}, o_{t-1})$, where $o_t$ denotes the environment observation (e.g., search results) following action $a_t$, and produces the next action $a_t$, a variable-length token sequence constituting a complete LLM response. The policy $\pi_\theta(a_t \mid h_t)$ is thus history-conditioned. An episode forms a trajectory $\tau = (a_1, o_1, \dots, a_T)$ and receives only a final outcome reward $R(\tau)$.

Most existing RL methods for agentic tasks assign this single outcome reward uniformly to all actions in the trajectory, implicitly treating each action as contributing equally, and in the same direction, toward the final result \cite{shao2024deepseekmath, jin2025search, ji2025tree}. More sophisticated approaches attempt finer-grained credit assignment by estimating intermediate state values through Monte Carlo sampling: for each intermediate history, multiple rollouts are generated and the resulting win rate is used as a value estimate~\cite{zhang2025process}. Despite their granularity, both families of methods share a common \textit{monotonicity assumption}: that a high-value intermediate state reliably leads to success, while a low-value state is unlikely to be  recovered, as if correctness accumulates predictably along the trajectory.

\begin{figure}[t]
    \centering
    \includegraphics[width=\linewidth]{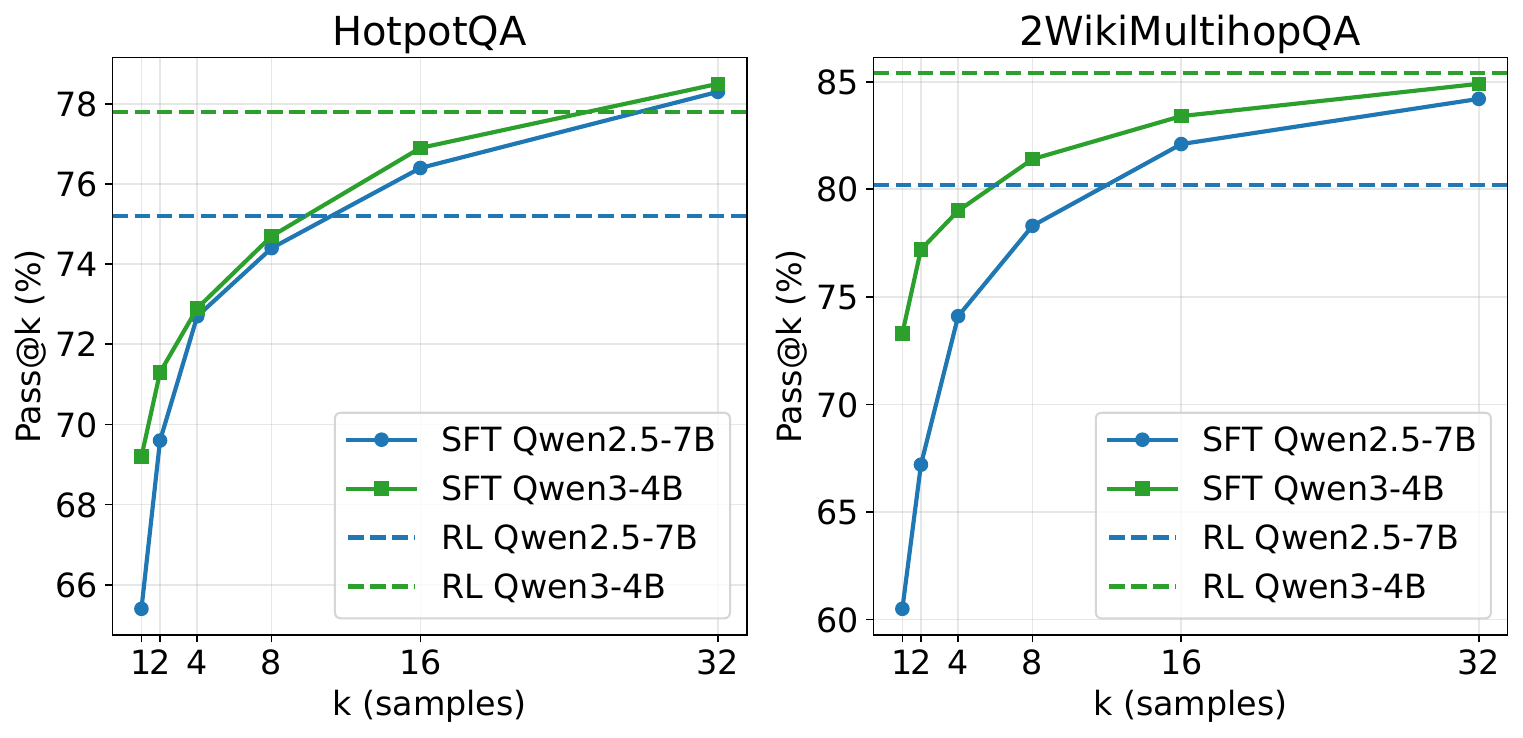}
    \caption{Pass@K results of SFT-trained Qwen2.5 and Qwen3 with partial resampling at the trajectory tail.}
    \label{fig:partial_pass_at_k}
      \vskip -0.2in
\end{figure}

However, multi-turn search trajectories violate this assumption. Following the ReAct paradigm~\cite{yao2022react}, a search agent alternates between issuing queries and synthesizing retrieved evidence. Unlike single-turn reasoning, multi-turn search continuously exposes the model to new observations, which may reshape its subsequent reasoning. As a result, trajectory correctness can be \textit{non-monotonic}: early mistakes may later be corrected through query reformulation, evidence reinterpretation, or continued exploration, while seemingly promising trajectories may still fail due to premature answering, hallucination, or faulty evidence aggregation.

To quantify this phenomenon, we conduct a partial resampling study on multi-hop QA benchmarks. Specifically, we evaluate 1K-example subsets of HotpotQA~\cite{yang2018hotpotqa} and 2WikiMultihopQA~\cite{ho-etal-2020-constructing}, using Qwen2.5-7B-Instruct~\cite{qwen2.5} and Qwen3-4B-Instruct~\cite{qwen3} models fine-tuned on search trajectories. For each example, we first sample a complete rollout, truncate it before the final answer generation step, and then resample 32 continuations from the same intermediate state. The model is allowed to either answer immediately or continue searching. Performance is evaluated using Pass@K with an LLM-as-a-Judge.

As shown in Figure~\ref{fig:partial_pass_at_k}, resampling from the same late-stage state improves Pass@32 by 10\%--25\% over Pass@1, indicating that different continuations from identical contexts can lead to divergent outcomes and that SFT-trained agents already possess a spontaneous ability to recover from imperfect states (see Appendix~\ref{app:tail_errors} for examples).

This observation exposes a key limitation of existing credit assignment strategies. In multi-turn search, the same history can lead to different outcomes depending on what the agent does next, making the history itself an insufficient unit of credit. Uniform outcome assignment is too coarse: in a failed trajectory, it penalizes all actions alike, including those that may have moved the agent toward the correct answer. Process rewards are more fine-grained, but they still assign scores to intermediate states rather than explicitly specifying which actions should be taken from them. Thus, neither approach explicitly supervises which behaviors under the current history lead to a correct final answer.

\subsection{Difficulty-Aware Branch Sampling}
\label{sec:branching_sampling}
Building on this observation, we shift the supervision target from evaluating how good a state is to comparing what continuations should be taken from that state. To obtain such localized signals, we introduce \emph{branch sampling}.

Given an initial trajectory, we truncate it at a late-stage state and resample a suffix continuation from the truncated point. This converts credit assignment from evaluating a single complete trajectory into comparing alternative decisions under the same prefix. To ensure informative training signal, we retain only \emph{contrastive branches}: if the original rollout is incorrect, we search for a correct continuation, and vice versa. Branches whose outcomes match the original are discarded, so that each retained pair provides a meaningful positive–negative contrast.

A fixed branching strategy is inefficient across tasks of varying difficulty. BranPO therefore adopts a \emph{difficulty-aware branch sampling} strategy. After the initial rollout, group accuracy is computed and branching budgets are assigned based on task accuracy and trajectory reward. For correct trajectories on easy tasks, \emph{simple branching} is applied: a single truncation-and-resample attempt is made. For hard tasks or incorrect trajectories, \emph{recursive branching} is used: starting from the last reasoning step, the method progressively backtracks the truncation point and resamples at each step, until either a contrastive branch is found or the branching budget is exhausted. As illustrated in Figure~\ref{fig:method}, this adaptive allocation concentrates computation on challenging instances where alternative continuations are most informative, while avoiding unnecessary exploration on trivial cases. The full algorithm is summarized in Appendix~\ref{sec:alg_detail}.

\subsection{Branching Advantage Estimation}
\label{sec:branpo}

We now formalize the proposed method, termed Branching Relative Policy Optimization (BranPO). BranPO builds on the agentic MDP defined in Section~\ref{sec:observation} and instantiates the branching strategy introduced above through a principled advantage estimation scheme.

Following GRPO, BranPO treats each task prompt $q$ as a group and samples $N$ trajectories $\{\tau^{(q,j)}\}_{j=1}^N$ from the current policy.
Branching is performed independently for each trajectory. For the $n$-th trajectory, we truncate it at a selected turn $k$ to obtain
a shared prefix
\begin{equation}
    \mathcal{B}_{q,n}
    \;=\;
    \tau^{(q,n)}_{\le k} = (a_1,o_1\dots,a_{k},o_k)
\end{equation}
which defines a state $s_{k+1}$ in the agentic MDP. From this state, we sample a set of branch continuations
\begin{equation}
    \mathcal{C}_{q,n}
    \;=\;
    \big\{ \tau^{(q,n,i)}_{>k} \big\}_{i=1}^{|\mathcal{C}_{q,n}|}
\end{equation}Each branch yields a complete trajectory with outcome reward $
    r_{\text{branch}}^{(q,n,i)} = R(\mathcal{B}_{q,n}, \mathcal{C}_{q,n}^{(i)})$.

When both successful and failed branches originate from the same prefix, credit assignment for the shared prefix becomes ambiguous. To mitigate this ambiguity and reduce gradient variance, we define the base trajectory reward for the prefix as the mean return of its branches:

\begin{equation}
    r_{\text{base}}^{(q,n)}
    \;=\;
    \frac{1}{|\mathcal{C}_{q,n}|}
    \sum_{i=1}^{|\mathcal{C}_{q,n}|}
    r_{\text{branch}}^{(q,n,i)}
\end{equation}

BranPO normalizes advantages following GRPO. The base advantage assigned to the shared prefix $\mathcal{B}_{q,n}$ is computed across trajectories sampled from the same prompt:
\begin{equation}
    \hat{A}_{\text{base}}^{(q,n)}
    =
    \frac{
        r_{\text{base}}^{(q,n)} - \mu_r^{\text{base}}(q)
    }{
        \sigma_r^{\text{base}}(q)
    }
\end{equation}
where $\mu_r^{\text{base}}(q)$ and $\sigma_r^{\text{base}}(q)$ is the mean and standard deviation of the group $\{r_{\text{base}}^{(q,j)} \mid 1 \le j \le N \}$.

For each branch continuation $\mathcal{C}_{q,n}^{(i)}$, the advantage is computed across all branches from the same prompt:
\begin{equation}
    \hat{A}_{\text{branch}}^{(q,n,i)}
    =
    \frac{
        r_{\text{branch}}^{(q,n,i)} - \mu_r^{\text{branch}}(q)
    }{
        \sigma_r^{\text{branch}}(q)
    }
\end{equation}
with $\mu_r^{\text{branch}}(q)$ and $\sigma_r^{\text{branch}}(q)$ is the mean and standard deviation of the group $\left\{ r_{\text{branch}}^{(q,j,l)} \mid 1 \le j \le N, \, 1 \le l \le |\mathcal{C}_{q,j}| \right\}$.

The resulting policy gradient decomposes credit assignment across the trajectory.  Specifically, the base advantage $\hat{A}_{\text{base}}^{(q,n)}$ is applied to all actions in the shared prefix $\mathcal{B}_{q,n}$ through the log probability $\log p_\theta(\mathcal{B}_{q,n})$. For each branch continuation, the corresponding branch-level advantage $\hat{A}_{\text{branch}}^{(q,n,i)}$ is applied to the actions in the suffix via the conditional log probability $\log p_\theta(\mathcal{C}^{(i)}_{q,n} \mid \mathcal{B}_{q,n})$.

\paragraph{Why BranPO is more effective?}
The structure of BranPO naturally unifies trajectory-level RL with contrastive preference optimization within a single framework. As derived in Appendix~\ref{sec:theory}, BranPO assigns credit differently to prefixes and suffixes.

For the shared prefix $\mathcal{B}$ (the search history), it functions as a variance-reduced variant of GRPO. By averaging returns across branched continuations, it stabilizes the learning signal for early steps that face inherently ambiguous credit assignment (Appendix~\ref{app:variance}).

For the branched suffixes, BranPO induces a contrastive gradient structure analogous to preference optimization. By contrasting a correct continuation $\mathcal{C}^+$ against an incorrect one $\mathcal{C}^-$ from the same prefix, the suffix gradient takes the form:
\begin{equation}
\small
    \nabla_\theta \mathcal{L}_{\text{suffix}} 
    \;\propto\; 
    \underbrace{\nabla_\theta \log \pi_\theta(\mathcal{C}^+ \mid \mathcal{B}) - \nabla_\theta \log \pi_\theta(\mathcal{C}^- \mid \mathcal{B})}_{\text{Contrastive preference optimization}}
\normalsize
\end{equation}
This gradient explicitly increases the likelihood gap between successful and failed continuations conditioned on the same history, sharing the core mechanism of DPO methods~\cite{rafailov2023direct} that optimize pairwise preferences. The key distinction is that BranPO derives this contrastive signal online from the agent's own branched rollouts rather than from a fixed preference dataset, and applies it selectively to the suffix where divergence actually occurs. This hybrid structure allows BranPO to maintain robust optimization for early-stage retrieval actions through group-wise trajectory comparison, while leveraging the sample efficiency of contrastive supervision for late-stage synthesis decisions where credit can be precisely localized.

\begin{table*}[ht]
\centering
\resizebox{0.9\textwidth}{!}{
\begin{tabular}{lcccccccccc}
\toprule
 
& \multicolumn{2}{c}{2WikiMQA} 
& \multicolumn{2}{c}{HotpotQA} 
& \multicolumn{2}{c}{MuSiQue} 
& \multicolumn{2}{c}{Bamboogle} 
& \multicolumn{2}{c}{Average} \\
\cmidrule(lr){2-3} \cmidrule(lr){4-5} \cmidrule(lr){6-7} \cmidrule(lr){8-9} \cmidrule(lr){10-11}
Method & EM & F1 & EM & F1 & EM & F1 & EM & F1 & EM & F1 \\
\midrule
\multicolumn{11}{c}{\texttt{\textbf{Outcome Reward}}} \\
\midrule
Search-R1$^\dagger$ \cite{jin2025search}  & 31.2 & 37.6 & 39.4 & 50.2 & 18.1 & 26.2 & 38.4 & 50.1 & 31.8 & 41.0 \\

SE-Search$^\dagger$ \cite{li2026se} & 35.5 & 42.2 & 43.8 & 55.9 & 19.5 & 29.0 & 48.8 & 60.1 & 36.9 & 46.8 \\
\midrule
\multicolumn{11}{c}{\texttt{\textbf{Process Reward}}} \\
\midrule
ReasonRAG$^\dagger$ \cite{zhang2025process}     & 43.6 & 50.4 & 38.4 & 48.9 & 12.8 & 20.6 & 36.0 & 45.5 & 32.7 & 41.3 \\
StepSearch$^\dagger$ \cite{wang2025stepsearch}  & 36.6 & 43.1 & 38.6 & 50.2 & \textbf{22.6} & 31.2 & 40.0 & 53.4 & 34.5 & 44.5 \\
GiGPO \cite{feng2025groupingroup} & 42.3 & 48.0 & 42.2 & 53.4 & 19.2 & 27.1 & 44.0 & 50.1 & 36.9 & 44.7 \\
CriticSearch$^\dagger$ \cite{zhang2025criticsearch} & 42.8 & 50.1 & 44.2 & 56.0 & 19.4 & 28.1 & 47.2 & 59.2 & 38.4 & 48.4 \\
TIPS$^\dagger$ \cite{xie2026tips}  & 43.0 & 50.6 & 43.0 & 54.7 & 17.1 & 26.6 & 36.8 & 52.2 & 35.0 & 46.0\\
\midrule
\multicolumn{11}{c}{\texttt{\textbf{Branch Sampling}}} \\
\midrule
Tree-GRPO \cite{ji2025tree} & 41.9 & 48.6 & 44.0 & 56.2 & 20.9 & 30.3 & 44.0 & 57.1 & 37.7 & 48.1\\
ARPO \cite{dong2025arpo} & 42.4 & 51.3 & 46.1 & 59.1 & 18.2 & 27.3 & 40.0 & 50.9 & 36.7 & 47.2 \\
AEPO \cite{dong2025aepo} & 46.7 & 53.5 & 45.7 & 58.2 & 19.3 & 28.4 & 44.0 & 54.6 & 38.9 & 48.7 \\
\rowcolor{green!20}
BranPO (Ours)& \textbf{48.1} & \textbf{56.7} & \textbf{48.9} & \textbf{61.8} & 21.2 & \textbf{32.4} & \textbf{50.4} & \textbf{62.0} & \textbf{42.2} & \textbf{53.2}\\
\bottomrule
\end{tabular}
}
\caption{ Performance comparison on multi-hop QA benchmarks. All methods except StepSearch are trained on Search-R1 training set, with HotpotQA as the in-domain test set. StepSearch is trained on the MuSiQue training set. $^\dagger$ means the results are obtained from the original paper. The best results are indicated with \textbf{bold}.}
\label{tab:main_results_7b}
\vskip -0.1in
\end{table*}

\begin{table*}[t]
\centering
\small
\centering
\resizebox{0.8\textwidth}{!}{%
\begin{tabular}{l 
                c c c c c c c c c c             
                }
\toprule

& \multicolumn{2}{c}{2WikiMQA}
& \multicolumn{2}{c}{HotpotQA}
& \multicolumn{2}{c}{MuSiQue}
& \multicolumn{2}{c}{Bamboogle}
& \multicolumn{2}{c}{Average}

 \\
\cmidrule(lr){2-3} \cmidrule(lr){4-5} \cmidrule(lr){6-7} \cmidrule(lr){8-9} \cmidrule(lr){10-11}

 Method & F1 & LasJ & F1 & LasJ & F1 & LasJ & F1 & LasJ & F1 & LasJ
 \\
\midrule
\multicolumn{11}{c}{\texttt{\textbf{Qwen2.5-7B-Instruct}}} \\
\midrule
GRPO & 64.8 & 73.6 & 60.4 & 71.9 & 34.6 & 42.0 & 51.3 & 60.8 & 52.8 & 62.1  \\
Tree-GRPO & 67.6 & 75.1 & 58.3 & 69.4 & 34.6 & 42.0 & 49.6 & 57.6 & 52.5 & 61.0 \\
ARPO & 63.3 & 74.3 & 58.6 & 68.9 & 31.7 & 40.0 & 51.8 & 59.2 & 51.4 & 60.6 \\
AEPO & 65.1 & 76.0 & 60.4 & 71.7 & 32.5 & 40.4 & \textbf{53.4} & \textbf{61.6} & 52.9 & 62.4  \\
\rowcolor{green!20}
BranPO & \textbf{72.0} & \textbf{80.2} & \textbf{62.7} & \textbf{75.2} & \textbf{35.1} & \textbf{43.8} & 52.7 & \textbf{61.6} & \textbf{55.6} & \textbf{65.2} \\
\midrule
\multicolumn{11}{c}{\texttt{\textbf{Qwen3-4B-Instruct-2507}}} \\
\midrule
GRPO & 71.9 & 82.6 & 61.4 & 74.5 & 37.3 & 44.9 & 58.3 & 70.4 & 57.2 & 68.1\\
Tree-GRPO & 68.7 & 78.2 & 58.4 & 71.2 & 32.0 & 41.0 & 54.6 & 65.6 & 53.4 & 64.0 \\
ARPO & 67.0 & 77.0 & 57.6 & 73.4 & 31.9 & 39.8 & 51.8 & 64.0 & 52.1 & 63.6  \\
AEPO & 69.5 & 78.9 & 60.6 & 72.0 & 32.5 & 41.2 & 56.8 & 68.0 & 54.9 & 65.0 \\
\rowcolor{green!20}
BranPO   & \textbf{75.6} & \textbf{85.4} & \textbf{64.5} & \textbf{77.8} & \textbf{38.9} & \textbf{48.7} & \textbf{60.3} & \textbf{72.0} & \textbf{59.8} & \textbf{71.0}  \\

\bottomrule
\end{tabular}%
}
\caption{Performance on multi-hop QA tasks when training on ASearcher dataset, with F1 and LLM-as-a-Judge (LasJ) score on each dataset and average. The best results are indicated with \textbf{bold}.}
\label{tab:main_result}
\vskip -0.1in
\end{table*}

\subsection{Redundant Step Masking}

Redundant tool usage is a common failure mode in multi-turn search agents \cite{team2025mirothinker}. Agents may continue to invoke the search tool even after sufficient evidence has been obtained, resulting in redundant actions with little additional benefit. An example is shown in Appendix~\ref{app:redundant}. Under the branch sampling strategy introduced in Section~\ref{sec:branching_sampling}, this issue can be further amplified. When truncation occurs near the trajectory tail, continuations that keep searching may be preferentially sampled if intermediate evidence appears incomplete, biasing training toward excessive verification.

While additional search steps can be reasonable for difficult instances, uncontrolled redundancy leads to inefficient training and suboptimal policies. To address this issue, we propose \emph{Redundant Step Masking (RSM)}, which augments BranPO with a dedicated redundancy-detection mechanism.
For correct trajectories whose interaction length exceeds the average length of correct trajectories within the same group, instead of applying the default BranPO branching strategy, we explicitly branch from the penultimate model response. If an alternative continuation reaches a correct answer with fewer steps, the original extra step is identified as redundant.

During optimization, the advantages associated with these redundant actions are masked and excluded from gradient computation, while the shorter correct continuation is retained as a valid training branch. Specifically, for branches identified as containing redundant steps, we apply a mask to the branch advantage:
\begin{equation}
    \hat{A}_{\text{branch}} = \hat{A}_{\text{branch}} \cdot \mathrm{Mask}_{\text{RS}}
\end{equation}
where $\mathrm{Mask}_{\text{RS}} \in \{0,1\}$ zeros out the advantages corresponding to redundant actions. By decoupling correctness from efficiency at the step level, RSM prevents reinforcement of unnecessary tool usage while preserving informative contrastive learning signals, which is beneficial for stable optimization in long-horizon agentic settings.

\section{Experiments}

\paragraph{Datasets.}
We train BranPO on two widely used search datasets: Search-R1 \cite{jin2025search} and ASearcher \cite{gao2025beyond} . Search-R1 training set mainly consists of 1--2 hop search questions from NQ \cite{kwiatkowski2019natural} and HotpotQA \cite{yang2018hotpotqa}. ASearcher contains more challenging multi-hop search tasks, allowing us to assess the performance of our algorithm in longer-horizon scenarios. We filter out invalid questions from the original ASearcher dataset. For each protocol, we keep the corresponding search engine and corpus consistent with the original work. 

\paragraph{Training.}
For experiments following Search-R1, we keep the training configuration as close as possible to the original Search-R1 setup. For the ASearcher benchmark, we adopt a two-stage training strategy following AgentGym-RL \cite{xi2025agentgym}: we first limit the interaction horizon to 4 steps and then extend it to 8 steps. Given our computational budget, this curriculum is sufficient to solve the tasks efficiently while stabilizing long-horizon RL training. We use a fully on-policy training setup rather than the asynchronous GRPO used in ASearcher, which allows us to isolate the effect of different sampling and credit-assignment strategies.

 For training stability, we warm-start all RL-based methods from the same SFT-initialized model; details of the SFT procedure are provided in the Appendix~\ref{sec:cold_start}. Full RL training details are provided in Appendix~\ref{sec:rl_details}, and algorithm settings are described in Appendix~\ref{sec:alg_detail}. As shown in Appendix~\ref{sec:training_efficiency}, BranPO achieves comparable wall-clock training time to other methods, ensuring fair comparison and comparable training cost.

\paragraph{Evaluation.}
We evaluate on four multi-hop QA benchmarks: HotpotQA \cite{yang2018hotpotqa}, 2WikiMultihopQA \cite{ho-etal-2020-constructing}, MuSiQue \cite{trivedi2022musique}, and Bamboogle \cite{press2023measuring}. Following the corresponding protocols, we report F1 score and an LLM-as-a-Judge (LasJ) metric for ASearcher, and Exact Match (EM) and F1 for Search-R1. We additionally evaluate web search ability on text-only questions from GAIA \cite{mialon2023gaia} in Appendix~\ref{app:web_eval}. Detailed evaluation protocols are provided in Appendix~\ref{sec:evaluation_details}.

\subsection{Main Results}
\label{sec:main_result}
We first evaluate BranPO under the Search-R1 setting with Qwen2.5-7B-Instruct~\cite{qwen2.5} as the backbone, comparing against outcome-reward RL methods, explicit process-reward RL methods, and prior branch-sampling methods. As shown in Table~\ref{tab:main_results_7b}, BranPO achieves the best average performance, reaching 42.2\% EM and 53.2\% F1, outperforming the strongest prior method by 3.3\% EM and 4.5\% F1 on average. These results suggest that branch-level preference optimization provides more effective credit assignment than previous methods, while avoiding the need for manually designed per-turn reward.

\paragraph{Results under the ASearcher setting.}
We further evaluate BranPO on the more challenging ASearcher dataset using Qwen2.5-7B-Instruct~\cite{qwen2.5} and Qwen3-4B-Instruct~\cite{qwen3}. In addition to GRPO, we focus on branch-sampling baselines, because these methods show strong performance in the shorter-horizon Search-R1 setting. During evaluation, we increase the maximum number of allowed search actions to 16 to assess longer-horizon behavior. The results are summarized in Table~\ref{tab:main_result}.

Across datasets and backbones, BranPO achieves the best F1 and LasJ scores, outperforming both GRPO and branch-sampling-based variants. While prior branch sampling methods do not consistently surpass GRPO on ASearcher, BranPO does, indicating that its gains come from branch-level preference optimization rather than branch sampling alone. This suggests that BranPO yields more stable credit assignment across extended decision sequences, enabling the model to sustain effective exploration under imperfect intermediate states in complex multi-hop search tasks.

\begin{figure*}[t]
    \centering
    \includegraphics[width=\textwidth]{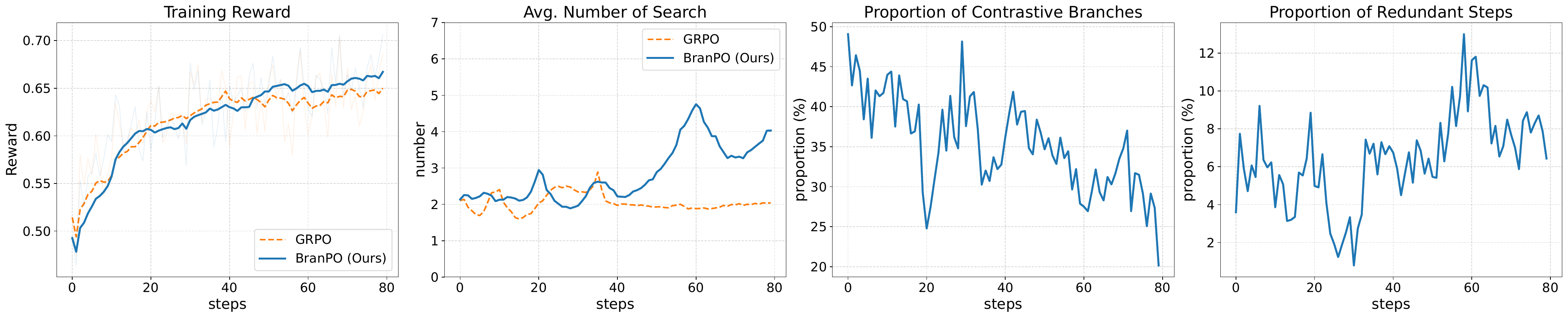}
    \caption{Training dynamics on Qwen2.5-7B-Instruct, including the training rewards, the change in the average number of search steps per trajectory, the proportion of trajectories that contain branches with different rewards, and the proportion of trajectories that are detected to include redundant steps under BranPO.}
    \label{fig:training}
      \vskip -0.1in
\end{figure*}

\paragraph{Training dynamics.} The two-stage training dynamics of Qwen2.5-7B-Instruct under ASearcher setting are illustrated in the Figure~\ref{fig:training}. It tracks the progression of training rewards, the evolution of the average number of search steps in training rollouts, the changing proportions of trajectories containing branches with different reward outcomes, and the fraction of trajectories that BranPO identifies as exhibiting redundant steps during training. As shown by the average number of search steps, BranPO learns faster than GRPO to increase search usage to ensure higher answer accuracy. Meanwhile, the proportion of trajectories with contrastive reward branches steadily decreases, indicating that late-stage errors are progressively reduced and that BranPO effectively optimizes tail decisions. As the interaction budget is extended, the proportion of trajectories containing redundant steps first increases and then declines, suggesting that RSM suppresses unnecessary exploratory behaviors. Nevertheless, trajectories with contrastive reward branches remain accounting for more than 20\% of all trajectories, indicating that substantial learning signals still remain for further optimization.

\subsection{Generalization to Extremely Long-Horizon Tasks}

When the interaction horizon becomes extremely long (e.g., 32+ steps), directly applying BranPO introduces considerable computational overhead due to recursive truncation and backtracking from the final step. To improve scalability, we introduce a segment-level variant of BranPO.

Given a trajectory with $M$ interaction steps, we evenly divide it into $S$ segments, each containing $M/S$ steps. BranPO then performs backtracking and branching at the segment level instead of every step, substantially reducing computational cost while preserving effective contrastive learning signals at key decision points.

We evaluate this variant on ALFWorld \cite{shridharalfworld}, a long-horizon embodied task benchmark in a text-based environment, following the GiGPO \cite{feng2025groupingroup} setting. We use Qwen2.5-7B-Instruct with a maximum interaction horizon of 50, $S=4$ segments, and a maximum recursive depth of 3. Results are shown in Table~\ref{tab:alfworld_result}. BranPO outperforms the strong baseline GiGPO by an average of 2.8 percentage points and achieves better performance than both GRPO and GiGPO on most subtasks, demonstrating strong generalization to other long-horizon domains.

\begin{table}[t]
\centering
\resizebox{\linewidth}{!}{
\begin{tabular}{lccccccc}
\toprule
Method & Pick & Look & Clean & Heat & Cool & Pick2 & All \\
\midrule
GRPO   & 90.8 & 66.1 & 89.3 & 74.7 & 72.5 & 64.7 & 77.6 \\
GiGPO  & 97.7 & 82.7 & \textbf{98.8} & 83.7 & \textbf{89.3} & 79.2 & 90.8 \\
BranPO & \textbf{99.0} & \textbf{89.7} & \textbf{98.8} & \textbf{95.9} & 88.0 & \textbf{86.1} & \textbf{93.6} \\
\bottomrule
\end{tabular}
}
\caption{Performance on ALFWorld benchmark.}
\label{tab:alfworld_result}
\end{table}

\begin{table}[t]
\centering
      \resizebox{\linewidth}{!}{
\begin{tabular}{lcccc}
\toprule
Method & 2WikiMQA & HotpotQA & MuSiQue & Avg. \\
\midrule
BranPO        & 69.4 & 59.8 & 34.5 & 54.6 \\
\quad w/o DBS & 67.8 & 59.4 & 30.8 & 52.7 \\
\quad w/o RSM & 68.2 & 60.0 & 33.0 & 53.7 \\
\bottomrule
\end{tabular}
}

\caption{Ablation study of BranPO components. DBS denotes for difficulty-aware branch sampling, RSM denotes for redundant step masking. The evaluation metric is F1 score.}
\label{tab:ablation}
\vskip -0.1in
\end{table}

\subsection{Ablation Study}
We perform ablation studies on Qwen3-4B-Instruct to assess the impact of Difficulty-aware Branch Sampling (DBS) and Redundant Step Masking (RSM) in BranPO. After the first-stage training, models are evaluated on three multi-hop QA benchmarks, with results reported in Table~\ref{tab:ablation}. Removing either DBS or RSM leads to performance degradation, demonstrating that both components are essential. DBS improves sampling efficiency by emphasizing informative trajectories, while RSM enhances step-level supervision by filtering redundant reasoning steps. Further discussion of RSM in extended horizon training is included in Appendix~\ref{app:rsm_analysis}.

\paragraph{Hyperparameter Sensitivity}
BranPO introduces two hyperparameters: the maximum number of branch samples per step ($B$) and the maximum number of backtracking steps ($R$). We study their effects under different settings, with detailed results provided in Appendix~\ref{app:hyperparam}. Even with a lightweight configuration ($B=4, R=1$), BranPO still consistently outperforms GRPO, while larger values further improve performance, indicating strong robustness and favorable scaling behavior.  

\section{Conclusion}

In this work, we present Branching Relative Policy Optimization (BranPO), an RL framework for long-horizon agentic tasks that addresses the credit assignment challenges posed by sparse rewards and the non-monotonic nature of correctness. BranPO constructs contrastive branches through tail-focused resampling and optimizes them with hierarchical group-relative advantages, transforming sparse outcome signals into localized preference supervision without step-level rewards or expensive tree expansion. Difficulty-aware branch sampling and Redundant Step Masking further improve efficiency and suppress continuation bias. Experiments show that BranPO achieves consistent gains over outcome-reward RL methods, explicit process-reward RL methods, and prior branch sampling methods on multi-hop QA, with generalization to 30+ step agentic tasks. Beyond empirical gains, this work highlights a practical insight: precise credit assignment is not essential; concentrating learning on outcome-divergent decisions under shared histories suffices for robust recovery from imperfect states.

\section*{Limitations}
Although BranPO provides more efficient updates than trajectory-level baselines, it does not fully resolve the sparse reward challenge. In extremely long-horizon tasks, generating rollouts to obtain outcome signals remains costly, and tail-focused resampling only partially alleviates this overhead. Integrating lightweight intermediate reward models or strategies that guide higher-quality rollouts to further reduce sampling cost is a promising direction for future work.

\bibliography{custom}

\appendix

\section{Implement Details}
\begin{algorithm*}[htbp]
\caption{BranPO update step with Difficulty-aware Branch Sampling}
\label{alg:branpo}
\begin{algorithmic}

\Require Policy $\pi_\theta$, reward $R(\cdot)$, initial rollout number $N$,
branch schedule $\textsc{Bran}(\cdot)$,
recursion schedule $\textsc{Recur}(\cdot)$

\Function{\textsc{InitialRollout}}{$g, N$}
        \State Sample trajectories $\{\tau^{(g)}_i\}_{i=1}^N \sim \pi_\theta$
        \State Calculate group accuracy $\text{acc}^{(g)}$ based on $\{\tau^{(g)}_i\}_{i=1}^N$
        \State \textbf{return} $(\{\tau^{(g)}_i\}_{i=1}^N, \text{acc}^{(g)})$
    \EndFunction
\newline
\Function{Branching}{$\tau^{(g)}_i, \text{acc}^{(g)}$}
        \For{$d = 1$ to $\textsc{Recur}(R(\tau_i^{(g)}), \text{acc}^{(g)})$}
            \State Select truncation index $k_d$ near tail $\mathcal{B}_{g,i} \gets \tau^{(g)}_{i, \le k_d}$
            \State $\mathcal{C}_{g,i} \gets (\tau^{(g)}_{i, > k_d}, R(\tau_i^{(g)}))$
            \For{$j = 1$ to $\textsc{Bran}(R(\tau_i^{(g)}), \text{acc}^{(g)})$}
                \State Sample $\tau_{>k_d}^{(g,i,j)} \sim \pi_\theta(\cdot \mid \mathcal{B}_{g,i})$
                \State $\tilde{\tau}^{(g,i,j)} \gets (\mathcal{B}_{g,i}, \tau_{>k_d}^{(g,i,j)})$
                \If{$R(\tilde{\tau}^{(g,i,j)}) \ne R(\tilde{\tau}^{(g,i)})$}
                    \State Add $(\tilde{\tau}^{(g,i,j)}, R(\tilde{\tau}^{(g,i,j)}))$ to $\mathcal{C}_{g,i}$
                    \State \textbf{return} $\mathcal{C}_{g,i}$
                \EndIf
            \EndFor
        \EndFor
        \State \textbf{return} $\mathcal{C}_{g,i}$
    \EndFunction
    
    \For{task $g$}
        \State $(\{\tau^{(g)}_i\}_{i=1}^N, \text{acc}^{(g)}) \gets \textsc{InitialRollout}(g, N)$

        \For{$i = 1$ to $N$}
            \State $\mathcal{C}_{g,i} \gets \textsc{Branching}(\tau^{(g)}_i, \text{acc}^{(g)})$

            \State $r_{\text{base}}^{(g,i)} \gets \dfrac{1}{|\mathcal{C}_{g,i}|} \sum_{(\tau, r) \in \mathcal{C}_{g,i}} r$

            \State $\hat{A}_{\text{base}}^{(g,i)} \gets \dfrac{r_{\text{base}}^{(g,i)} - \mu_r^{\text{base}}}{\sigma_r^{\text{base}}}$

            \For{branch $j \in \mathcal{C}_{g,i}$}
                \State $\hat{A}_{\text{branch}}^{(g,i,j)} \gets \dfrac{r_{\text{branch}}^{(g,i,j)} - \mu_r^{\text{branch}}(g,i)}{\sigma_r^{\text{branch}}(g,i)}$
            \EndFor

        \EndFor
    \EndFor

    \State Update $\theta$ using GRPO-style policy gradient with BranPO advantages
\end{algorithmic}

\end{algorithm*}

\subsection{Algorithm Details}
\label{sec:alg_detail}
The pseudocode for BranPO with Difficulty-aware Sampling is presented in Algorithm \ref{alg:branpo}. An F1 score of 0.8 is used as a threshold to determine whether an answer is considered correct during contrastive sampling. And task accuracy is calculated by the average F1 score of the initial rollouts in the same task group. In our experiments, the branch and recursion schedule criteria were selected to maximize the difference signals within the branches without increasing the total step time. Consequently, our branch sampling criteria are defined as follows:

\begin{equation}
    \text{Bran}(r, \text{acc}) = 
    \begin{cases} 
      1, & \text{if } \text{acc} = 1 \\
      2, & \text{otherwise}
    \end{cases}
\end{equation}

\begin{equation}
    \text{Recur}(r, \text{acc}) = 
    \begin{cases} 
      3, & \text{if } \text{acc} < 0.5 \text{ and } r < 0.8 \\
      2, & \text{if } \text{acc} < 1 \text{ or } r < 0.8 \\
      1, & \text{otherwise}
    \end{cases}
\end{equation}

which means:
\begin{itemize}
    \item If the task accuracy is equal to 1 (all correct), only the final step is truncated and one branch is sampled.
    \item If the task accuracy is below 0.5 and the trajectory-level F1 score is below 0.8, the trajectory is recursively truncated by moving backward from the last step to the third-to-last step, with two branches sampled at each truncation, until a correct branch is found.
    \item Otherwise, the trajectory is recursively truncated from the last step to the second-to-last step, with two branches sampled at each truncation, until a contrastive branch is found.
\end{itemize}
The selection of the branch and recursion budget hyperparameters above is primarily intended to ensure fairness in training time between the BranPO algorithm and other methods. We also experimented with additional hyperparameter settings to study the sensitivity of BranPO. Detailed results are provided in ~\ref{app:hyperparam}. The results show that even under lower training budgets, BranPO can still outperform other methods.

Regarding Redundant Mask detection, for samples where accuracy $> 0.5$ and $r=1$, we execute a truncation at the penultimate step. Subsequently, we branch 5 times until a more efficient step is identified.

\begin{table*}[h]
\centering
\begin{tabular}{l|l}
\toprule
Hyperparameter & Value \\
\midrule
Batch size & 256 \\
PPO mini-batch size & 256 \\
Max response length & 3,072 (4-turn limit), 7,168 (8-turn limit) \\
KL loss coefficient & 0.001 \\
KL type & k3 \\
Loss aggregate mode & token-mean \\
Rollout IS threshold & 2.0 \\
Rollout IS batch normalize & True \\
Rollout engine & vLLM \cite{kwon2023efficient} \\
Rollout temperature & 0.95 \\
Optimizer & AdamW \\
Learning rate & $5 \times 10^{-6}$ (ASearcher); $1\times10^{-6}$ (Search-R1) \\
Learning rate warm-up ratio & 0.0 \\
Adam $\beta_1$, $\beta_2$ & 0.9, 0.999 \\
Gradient clip norm & 1.0 \\
Training backend & FSDP \\
retrieval topk & 3 \\
\bottomrule
\end{tabular}
\caption{Training Hyperparameters.}
\label{tab:hyperparams}
\end{table*}

\subsection{Cold Start}
\label{sec:cold_start}
When conducting RL with Instruct models, we observed instability in the later stages of training. To mitigate this issue, we adopted a SFT to stabilize multi-turn training. The construction pipeline for the cold-start data is as follows. First, we perform zero RL training with GRPO using the Search-R1 \cite{jin2025search} training data, which consists of NQ \cite{kwiatkowski2019natural} and HotpotQA \cite{yang2018hotpotqa}, with the turn limit set to 4 and training for 80 iterations. Next, from the rollout data generated during training, we filter out the correct samples, randomly select 10,000 instances, and then reconstruct the tag format to obtain a multi-turn SFT dataset. For each model, the cold-start data are derived from its own zero RL training. Finally, we conduct SFT using LLaMA-Factory \cite{zheng2024llamafactory} with a learning rate of $1\times10^{-5}$, a batch size of 256, and train for 1 epoch.

\subsection{RL Training}
\label{sec:rl_details}
Our training framework is based on rLLM \cite{rllm2025}, an agentic RL training system built on veRL \cite{sheng2025hybridflow}. We introduce two fixes to the original rLLM framework:

\begin{itemize}
    \item \textbf{Token-in-token-out.}  
    In multi-turn training, each round's prompt includes the interaction history from previous rounds. In the original rLLM implementation, the prompt is re-encoded by the tokenizer at every turn. We observe that the re-tokenized sequence may not exactly match the tokens previously generated by the LLM, introducing off-policy effects that harm training stability. To address this issue, we cache the token IDs produced by the LLM at each turn and concatenate them with the existing prompt tokens and encoded environment tokens to form the next-turn prompt, avoiding repeated re-encoding and decoding.
    
    \item \textbf{Rollout Correction.}  
    This technique mitigates off-policy effects caused by distributional discrepancies between the inference engine and the training engine. We apply truncated token-level importance sampling to reduce this mismatch \cite{liu-li-2025-rl-collapse}.
\end{itemize}

During training, we use the F1 score as the primary reward. We also incorporate a format reward: if the output does not follow either
{\hypersetup{urlcolor=black}\url{<thinking>...</thinking><search>...</search>}} or
{\hypersetup{urlcolor=black}\url{<thinking>...</thinking><answer>...</answer>}}, or the model answer the question without one search, a reward of $-0.5$ is assigned to the entire trajectory.

The prompt template used in our experiments follows \cite{jin2025search} and is shown in Table~\ref{tab:template}. We replace \texttt{<think>} and \texttt{</think>} with \texttt{<thinking>} and \texttt{</thinking>} since \texttt{<think>} is a special token in Qwen3 series models. All other training hyperparameters are reported in Table~\ref{tab:hyperparams}.

\begin{table}[h]
\centering
\begin{tabularx}{\linewidth}{X}
\toprule
Answer the given question. You must conduct reasoning inside \think{and} first every time you get new information. After reasoning, if you find you lack some knowledge, you can call a search engine by \search{query} and it will return the top searched results between \info{and}. You can search as many times as your want. If you find no further external knowledge needed, you can directly provide the answer inside \answer{and}, without detailed illustrations. For example, \answer{Beijing}. Question: \textcolor{red}{question.}
 \\
\bottomrule
\end{tabularx}
\caption{Prompt Template.}
\label{tab:template}
\end{table}

\subsection{Evaluation Details}
\label{sec:evaluation_details}

\paragraph{Baselines}
We mainly compare BranPO with three categories of baseline reinforcement learning algorithms:

\begin{enumerate}
    \item \textbf{Outcome-reward-based RL methods}, including GRPO \cite{shao2024deepseekmath}, Search-R1 \cite{jin2025search}, and SE-Search \cite{li2026se}.
    
    \item \textbf{Explicit turn-level process reward methods}, including ReasonRAG \cite{zhang2025process}, StepSearch \cite{wang2025stepsearch}, GiGPO \cite{feng2025groupingroup}, CriticSearch \cite{zhang2025criticsearch}, and TIPS \cite{xie2026tips}.
    
    \item \textbf{Branch-sampling-based GRPO methods}, including Tree-GRPO \cite{ji2025tree} and entropy-guided branching strategies ARPO \cite{dong2025arpo} and AEPO \cite{dong2025aepo}.
\end{enumerate}

These three categories cover the vast majority of algorithmic improvements in agentic RL. Due to computational resource constraints, under the ASearcher setting, we only compare against GRPO and branch-sampling-based GRPO methods, as these methods generally exhibit stronger performance in short-horizon scenarios. 

Our training data follows ASearcher \cite{gao2025beyond}; however, ASearcher still adopts the GRPO algorithm, with the number of interaction rounds extended to 32. Nevertheless, in the ASearcher dataset, the vast majority of problems can already be solved within 8 interaction steps. We therefore believe that extending the interaction horizon to 32 steps on this dataset is unnecessary, as it mainly introduces longer training time and may encourage redundant search behaviors.
Therefore, we do not compare directly with ASearcher, but instead compare with the results of GRPO under our training budget.
    

    

\paragraph{Evaluation Data}
We adopt the exact same evaluation data and sampling protocol as in ASearcher \cite{gao2025beyond}. Specifically, our evaluation suite includes:
\begin{itemize}
    \item \textbf{Multi-hop QA datasets}: 2WikiMultihopQA \cite{ho-etal-2020-constructing} (12,576 examples), HotpotQA \cite{yang2018hotpotqa} (7,405 examples), MuSiQue \cite{trivedi2022musique} (2,417 examples), and Bamboogle \cite{press2023measuring} (125 examples);

    \item \textbf{Web QA dataset}: GAIA \cite{mialon2023gaia} (103 text-only instances).
\end{itemize}
Under the ASearcher setting, for all datasets except Bamboogle and GAIA, we only use the randomly sampled 1,000 examples adopted in ASearcher \cite{gao2025beyond}. 

\paragraph{Evaluation Settings}
During evaluation, we use the following generation parameters: temperature = 0.7 and top\_p = 0.8.

\begin{table}[h]
\centering 
\resizebox{\linewidth}{!}{
\begin{tabular}{lcccc}
\toprule
 & \multicolumn{3}{c}{Time Composition (s)} & Total (s) \\
\cmidrule(lr){2-4}
Method 
& Rollout 
& \begin{tabular}[c]{@{}c@{}}Compute\\ log\_prob\end{tabular} 
& \begin{tabular}[c]{@{}c@{}}Update\\ actor\end{tabular} 
&  \\
\midrule
GRPO   & 108 & 75 & 251 & 439 \\
Tree-GRPO   & 103 & 71 & 248 & 424 \\
BranPO & 191 & 55 & 182 & 429 \\
\bottomrule
\end{tabular}
}
\caption{Wall-clock time analysis of GRPO and BranPO. We report the average time over the first five training steps.}
\label{tab:time_analysis}
\end{table}

\subsection{Training Efficiency}
\label{sec:training_efficiency}

To ensure a fair comparison between GRPO, Tree-GRPO and BranPO, we analyze their wall-clock training time under a single training round. We report the average time of the first five training steps of Qwen2.5-7B-Instruct at Table~\ref{tab:time_analysis}, since later-stage training time is strongly affected by the average number of environment interactions per trajectory. 

As shown in Table~\ref{tab:time_analysis}, compared to GRPO and Tree-GRPO, BranPO introduces additional rollout overhead due to branching-based trajectory construction. However, this increased rollout cost is largely offset by a reduced policy update time, as BranPO requires fewer effective training samples for optimization. As a result, the overall wall-clock training time of BranPO remains comparable to that of GRPO and Tree-GRPO.

For ARPO and AEPO, we adopt their original rollout configurations to fully exploit the benefits of their entropy-guided branching strategies: the initial rollout size is set to 8, and the final number of trajectories is expanded to 16. In contrast, the rollout setting of GRPO in Table~\ref{tab:time_analysis} uses only 8 trajectories. Therefore, although ARPO and AEPO incur substantially higher computational costs than GRPO, Tree-GRPO, and BranPO, their performance still remains inferior to our method.

\begin{table*}[htbp]
\centering

\begin{tabular}{lcccc}
\toprule
& 2WikiMQA & HotpotQA & MuSiQue & Avg. \\
\midrule
GRPO & 62.4 & 58.6 & 31.1 & 50.7 \\
Tree-GRPO & 63.9 & 56.7 & 31.9 & 50.8 \\

BranPO & & & & \\
\quad No branch ($B$=0,$R$=1) & 58.6 & 56.8 & 30.1 & 48.5 \\
\quad $B$=2, $R$=1 & 61.9 & 58.3 & 31.7 & 50.6 \\
\quad $B$=4, $R$=1 & 65.9 & 59.7 & \textbf{33.7} & 53.1 \\
\quad $B$=2, $R$=2 & 64.2 & 58.8 & 31.5 & 51.5 \\
\quad $B$=4, $R$=2 & \textbf{68.0} & 60.7 & 32.5 & 53.7 \\
\quad $B$=2, $R$=3 & 67.7 & 60.9 & 32.2 & 53.6 \\
\quad $B$=4, $R$=3 & 67.9 & \textbf{61.1} & 32.5 & \textbf{53.8} \\

\bottomrule
\end{tabular}
\caption{Performance under different branching and recursion budgets.}
\label{tab:hyperparam}
\end{table*}

\section{Generalize to web scenario.} 
\label{app:web_eval}

\begin{table*}[htbp]
\centering
\resizebox{0.6\textwidth}{!}{%
\begin{tabular}{l 
                c c c c c c c c c      
                }
\toprule


& \multicolumn{2}{c}{Lv.1}
& \multicolumn{2}{c}{Lv.2}
& \multicolumn{2}{c}{Lv.3}
& \multicolumn{2}{c}{All} \\
\cmidrule(lr){2-3} \cmidrule(lr){4-5} \cmidrule(lr){6-7} \cmidrule(lr){8-9}
 Method  & F1 & LasJ & F1 & LasJ & F1 & LasJ & F1 & LasJ \\
\midrule
\multicolumn{9}{c}{\texttt{\textbf{Qwen2.5-7B-Instruct}}} \\
\midrule
GRPO  & 38.2 & 38.5 & 19.3 & 22.2 & \textbf{6.0} & 0.0 & 24.9 & 25.8 \\
BranPO & \textbf{41.2} & \textbf{41.1} & \textbf{26.9} & \textbf{28.4}  & 3.6 & 0.0 & \textbf{29.6} & \textbf{29.9} \\
\midrule
\multicolumn{9}{c}{\texttt{\textbf{Qwen3-4B-Instruct-2507}}} \\
\midrule
GRPO  & 53.5 & 57.1 & 36.8 & 38.5 & \textbf{6.1} & \textbf{2.1} & 39.5 & 41.3 \\
BranPO   & \textbf{57.9} & \textbf{61.6} & \textbf{36.9} & \textbf{40.9} & 4.8 & \textbf{2.1} & \textbf{41.1} & \textbf{44.2} \\
\bottomrule
\end{tabular}
}
\caption{Avg@4 performance on GAIA with real web search engine. The best results are indicated with \textbf{bold}.}
\label{tab:web_result}
\end{table*}

We further evaluate BranPO in a real-world web search scenario to assess its generalization ability, with results reported in Table~\ref{tab:web_result}. We evaluate our methods on the GAIA benchmark \cite{mialon2023gaia}, reporting the Avg@4 metric which averages the F1 scores and LLM-as-a-Judge scores across four independent runs. During evaluation, each search retrieves up to 10 web pages, which are then processed by a summarization model instantiated with Qwen3-30B-A3B-Instruct.

On the GAIA benchmark, BranPO consistently outperforms GRPO on both Qwen2.5-7B-Instruct and Qwen3-4B-Instruct, demonstrating strong generalization to real-world web search tasks. 

For LLM-as-a-Judge, the prompt we adopt follows ASearcher \cite{gao2025beyond}, as shown in Table~\ref{tab:lasj_prompt}.

\begin{table*}[htbp]
\centering
\begin{tabularx}{\linewidth}{X}
\toprule
You are an evaluation assistant. Please determine if the predicted answer is equivalent to any of the labeled answers.
\\ \\
Question: \textcolor{red}{question}
\\ \\
Labeled Answers: \textcolor{blue}{gt\_answers\_str}
\\ \\
Predicted Answer: \textcolor{blue}{pred\_answer}
\\ \\
Did the model give an answer \textbf{equivalent} to any of the labeled answers? Respond with "Correct" if equivalent to at least one, otherwise "Incorrect".
\\ \\
Output in JSON format:
\\ \\
```json
\\
\{
\\
    "rationale": "your rationale",
\\
    "judgement": "Correct" or "Incorrect"
\\
\}
\\
```
 \\
\bottomrule
\end{tabularx}
\caption{LLM-as-a-Judge Prompt Template.}
\label{tab:lasj_prompt}
\end{table*}

To provide the model with cleaner and more relevant information during web search tasks, we employ a summarization step after retrieving web pages. Specifically, we use a dedicated summarization model to extract query-relevant content from each retrieved webpage. Across all evaluations, we use \texttt{Qwen3-30B-A3B-Instruct} as the summarization model, with the prompt template shown in Table~\ref{tab:summarization_prompt}.

\begin{table*}[htbp]
\centering
\begin{tabularx}{\linewidth}{X}
\toprule
You are an AI research assistant. Your task is to analyze the provided webpage content and extract information relevant to the user's query.
\\ \\
Search Query: \textcolor{red}{search\_query} 
\\\\
Webpage Content: \textcolor{blue}{page\_content}
\\\\
Please analyze the content and return a \textbf{valid JSON object} with the following fields:
\\
- "is\_helpful": (boolean) Whether the content contains information relevant to the user's query.
\\
- "summary": (string) A concise, information-dense summary of the content relevant to the query. If 'is\_helpful' is false, briefly explain why (e.g., "Content is behind a paywall" or "Topic unrelated").
\\
- "relevance\_score": (integer) A score from 0-10 indicating how relevant this page is.
\\\\
Ensure the output is raw JSON, without Markdown formatting (like ```json).
 \\
\bottomrule
\end{tabularx}
\caption{Web page summarization Prompt Template}
\label{tab:summarization_prompt}
\end{table*}

\section{Hyperparameter Analysis}
\label{app:hyperparam}

As introduced in Section~\ref{sec:alg_detail}, BranPO involves two key hyperparameters: the maximum branching factor ($B$) and the maximum recursion depth ($R$). Specifically, $B$ corresponds to the upper bound of extra branch samples at each truncation step, while $R$ denotes the maximum number of backtracking steps allowed in the recursive truncation process. In practice, these values correspond to the largest outputs of the piecewise functions \text{Bran}(r, \text{acc}) and \text{Recur}(r, \text{acc}).

To investigate their impact, we evaluate BranPO under different ($B$, $R$) configurations , based on the \texttt{Qwen2.5-7B-Instruct} model after the first-stage training with ASearcher dataset.The results are shown in Table~\ref{tab:hyperparam}.

From the results, we draw two main observations. First, even with a reduced recursion depth, such as $B$=4, $R$=1, BranPO consistently outperforms both GRPO and Tree-GRPO, indicating that strong gains can be achieved with limited backtracking. Second, BranPO demonstrates a clear scaling pattern on multi-hop reasoning tasks: increasing either the branching factor or recursion depth generally leads to improved performance, suggesting that richer exploration of alternative trajectories provides more informative learning signals.

Overall, these results highlight the robustness of BranPO to hyperparameter choices and its ability to effectively trade off computation and performance.

\section{Connections between GRPO, DPO, and BranPO}
\label{sec:theory}

This appendix provides a compact theoretical account of (i) the exact relationship between BranPO and GRPO, (ii) why BranPO reduces variance in prefix credit assignment, and (iii) why BranPO's suffix optimization admits a preference-style (DPO-like) interpretation that can be more sample-efficient.

\subsection{Setup and Notation}
We consider an episodic agentic MDP with trajectories $\tau = (a_1,\dots,a_T)$ induced by a policy $\pi_\theta$.
Let $R(\tau)\in\{0,1\}$ denote the outcome reward.
Let $\pi_{\mathrm{ref}}$ be a fixed reference policy and $\beta>0$ the KL coefficient.

For a truncation index $k$, define the trajectory prefix and suffix as
\begin{align}
    \mathcal{B} &\triangleq \tau_{\le k} = (a_1,\dots,a_k),
\\
\mathcal{C} &\triangleq \tau_{>k} = (a_{k+1},\dots,a_T)
\end{align}We write concatenation as $(\mathcal{B},\mathcal{C})$.

BranPO samples $M$ suffix continuations $\{\mathcal{C}^{(i)}\}_{i=1}^M$ from the conditional policy $\pi_\theta(\cdot\mid \mathcal{B})$ and evaluates branch rewards
\begin{align}
    r_{\mathrm{branch}}^{(i)} &\triangleq R(\mathcal{B},\mathcal{C}^{(i)}),
\\
r_{\mathrm{base}}(\mathcal{B}) &\triangleq \frac{1}{M}\sum_{i=1}^M r_{\mathrm{branch}}^{(i)}
\end{align}

Both GRPO and BranPO use the same KL-regularized policy optimization loss; BranPO differs only in the construction/assignment of advantages.
Abstracting away implementation details (e.g., token-level sums, masking), the shared structure is
\begin{equation}
\begin{multlined}
\mathcal{L}(\theta)
=
-\mathbb{E}\Big[\sum_{t=1}^T \hat{A}_t \log \pi_\theta(a_t\mid s_t)\Big]
\\+
\beta\;\mathbb{E}\Big[\sum_{t=1}^T \mathrm{KL}\big(\pi_\theta(\cdot\mid s_t)\,\|\,\pi_{\mathrm{ref}}(\cdot\mid s_t)\big)\Big]
\end{multlined}
\end{equation}

GRPO assigns a single trajectory-level advantage to all steps, while BranPO assigns a base advantage to prefix steps and branch advantages to suffix steps.
Group normalization (as in GRPO) applies an affine transform within each task group and therefore does not change the arguments below except for per-group scaling; we omit it when analyzing unbiasedness/variance and reintroduce it in the final statements.

\subsection{BranPO Reduces to GRPO}
\label{app:grpo_reduction}

\begin{theorem}[Reduction to GRPO]
\label{thm:branpo_to_grpo}
Assume BranPO samples exactly one continuation per prefix ($M=1$) and uses the unmodified rollout suffix (equivalently, no effective truncation/branching). Then $r_{\mathrm{base}}(\mathcal{B})=R(\tau)$ and the BranPO update is identical to GRPO, including the KL regularization term.
\end{theorem}

\begin{proof}
When $M=1$ and the continuation is the rollout suffix, we have $r_{\mathrm{branch}}^{(1)}=R(\tau)$ and thus
\begin{equation}
    r_{\mathrm{base}}(\mathcal{B})
=
\frac{1}{1} r_{\mathrm{branch}}^{(1)}
=
R(\tau)
\end{equation}

Hence the advantage assigned to prefix steps equals the usual trajectory return (up to the same within-group normalization used by GRPO), and there are no additional branch trajectories beyond the rollout itself.
Because the policy-gradient term and the KL term are unchanged by design, the resulting loss and gradient match GRPO exactly.
\end{proof}

 \paragraph{Takeaway.} In the absence of effective branches, trajectory prefixes are strictly correct or incorrect. Under these conditions, BranPO becomes equivalent to GRPO, demonstrating that it retains the same capacity to optimize the prefix.
 
\subsection{Prefix Variance Reduction}
\label{app:variance}

The core difficulty in multi-turn agent RL is that early actions in $\mathcal{B}$ may be correct but are penalized due to stochastic or fragile failures in the late suffix.
BranPO addresses this by estimating the \emph{conditional downstream value} of a prefix,
\begin{equation}
  Q^{\pi_\theta}(\mathcal{B})
\triangleq
\mathbb{E}_{u\sim \pi_\theta(\cdot\mid \mathcal{B})}\big[R(\mathcal{B},\mathcal{C})\big]  
\end{equation}
using multiple suffix samples and the sample mean $r_{\mathrm{base}}(\mathcal{B})$.

\begin{lemma}[Unbiased prefix value estimate without filtering]
\label{lem:unbiased_base}
If $\{u^{(i)}\}_{i=1}^M$ are i.i.d.\ samples from $\pi_\theta(\cdot\mid \mathcal{B})$, then
\begin{equation}
    \mathbb{E}\big[r_{\mathrm{base}}(\mathcal{B})\mid \mathcal{B}\big] = Q^{\pi_\theta}(\mathcal{B})
\end{equation}
\end{lemma}

\begin{proof}
Conditioned on $\mathcal{B}$, the branch rewards $\{r_{\mathrm{branch}}^{(i)}\}$ are i.i.d.\ with mean $Q^{\pi_\theta}(\mathcal{B})$.
Linearity of expectation yields
\begin{align}
    \mathbb{E}\big[r_{\mathrm{base}}(\mathcal{B})\mid \mathcal{B}\big]
&=
\frac{1}{M}\sum_{i=1}^M
\mathbb{E}\big[r_{\mathrm{branch}}^{(i)}\mid \mathcal{B}\big]
\\&=
Q^{\pi_\theta}(\mathcal{B})
\end{align}
\end{proof}

\begin{lemma}[Conditional variance reduction]
\label{lem:var_reduction}
Under the assumptions of Lemma~\ref{lem:unbiased_base},
\begin{align}
 \mathrm{Var}\big(r_{\mathrm{base}}(\mathcal{B})\mid \mathcal{B}\big)
&=
\frac{1}{M}\mathrm{Var}\big(R(\mathcal{B},\mathcal{C})\mid \mathcal{B}\big)
\\&\le
\frac{1}{4M}   
\end{align}
\end{lemma}

\begin{proof}
By i.i.d.\ sampling, the variance of the mean scales as $1/M$:
\begin{align}
\mathrm{Var}\big(r_{\mathrm{base}}(\mathcal{B})\mid \mathcal{B}\big)
&= \frac{1}{M^2}\sum_{i=1}^M
\mathrm{Var}\big(r_{\mathrm{branch}}^{(i)}\mid \mathcal{B}\big) \notag \\
&= \frac{1}{M}\mathrm{Var}\big(r_{\mathrm{branch}}^{(1)}\mid \mathcal{B}\big)
\end{align}
Since $R(\cdot)\in\{0,1\}$, $\mathrm{Var}(R)\le 1/4$, which gives the final inequality.
\end{proof}

Consider the (unclipped) REINFORCE-style prefix term for a fixed prefix $\mathcal{B}$:

\begin{align}
    G_{\mathrm{prefix}}
&\triangleq
r_{\mathrm{base}}(\mathcal{B})\;\nabla_\theta \log \pi_\theta(\mathcal{B}),
\\
\log \pi_\theta(\mathcal{B}) &\triangleq \sum_{t=1}^k \log \pi_\theta(a_t\mid s_t)
\end{align}

Conditioned on $\mathcal{B}$, the randomness in $G_{\mathrm{prefix}}$ arises solely from $r_{\mathrm{base}}(\mathcal{B})$, hence
\begin{equation}
\begin{multlined}
\mathrm{Var}\big(G_{\mathrm{prefix}}\mid \mathcal{B}\big) \\
= \|\nabla_\theta \log \pi_\theta(\mathcal{B})\|^2\;
\mathrm{Var}\big(r_{\mathrm{base}}(\mathcal{B})\mid \mathcal{B}\big)
\end{multlined}
\end{equation}
By Lemma~\ref{lem:var_reduction}, this conditional variance contracts by a factor $1/M$ compared to using a single rollout return. This is the formal mechanism by which branching stabilizes prefix credit assignment.

In practice, BranPO keeps only \emph{contrastive} branching outcomes. This selection changes the sampling distribution and generally introduces bias relative to $Q^{\pi_\theta}(\mathcal{B})$.
However, the same mechanism still reduces variance of the \emph{prefix learning signal} because (i) it aggregates multiple suffix outcomes, and (ii) it compresses the range of possible base returns.

To make the compression explicit, consider binary rewards $R(\tau)\in\{0,1\}$ and a contrastive rule that retains prefixes only when the retained branch set contains at least one success and at least one failure:
$$
\exists\, i,j \;\; \text{s.t.} \;\;
r_{\mathrm{branch}}^{(i)}=1,\;
r_{\mathrm{branch}}^{(j)}=0
$$

\begin{lemma}[Range compression under contrastive filtering]
\label{lem:range_compression}
If $R\in\{0,1\}$ and the retained branch set for a prefix $\mathcal{B}$ contains at least one $1$ and one $0$, then
$$
r_{\mathrm{base}}(\mathcal{B}) \in \Big[\frac{1}{M},\, 1-\frac{1}{M}\Big]
$$
and therefore, for any random variable supported on this interval,
$$
\mathrm{Var}\big(r_{\mathrm{base}}(\mathcal{B})\big)
\le
\frac{(1-2/M)^2}{4}
<
\frac{1}{4}
$$
\end{lemma}

\begin{proof}
With at least one success and one failure among $M$ Bernoulli samples, the number of successes is an integer in $\{1,\dots,M-1\}$, hence the mean is in $[1/M,1-1/M]$.
Any random variable supported on $[a,1-a]$ has variance at most $(1-2a)^2/4$ (achieved by a two-point distribution at the endpoints), yielding the bound with $a=1/M$.
\end{proof}

\paragraph{Takeaway.}
Unlike GRPO, which relies on a single binary outcome outcome ($\{0,1\}$) per rollout, contrastive branching keeps the base reward $r_{\mathrm{base}}(\mathcal{B})$ away from extremes. This mitigates large-magnitude, high-variance prefix updates driven by tail stochasticity, effectively stabilizing training without requiring expensive Monte Carlo sampling.This mechanism aligns directly with BranPO's design goal: to efficiently attenuate ambiguous credit assignment in scenarios where a single prefix can lead to both success and failure.
\begin{figure*}[t]
    \centering
    \includegraphics[width=0.5\textwidth]{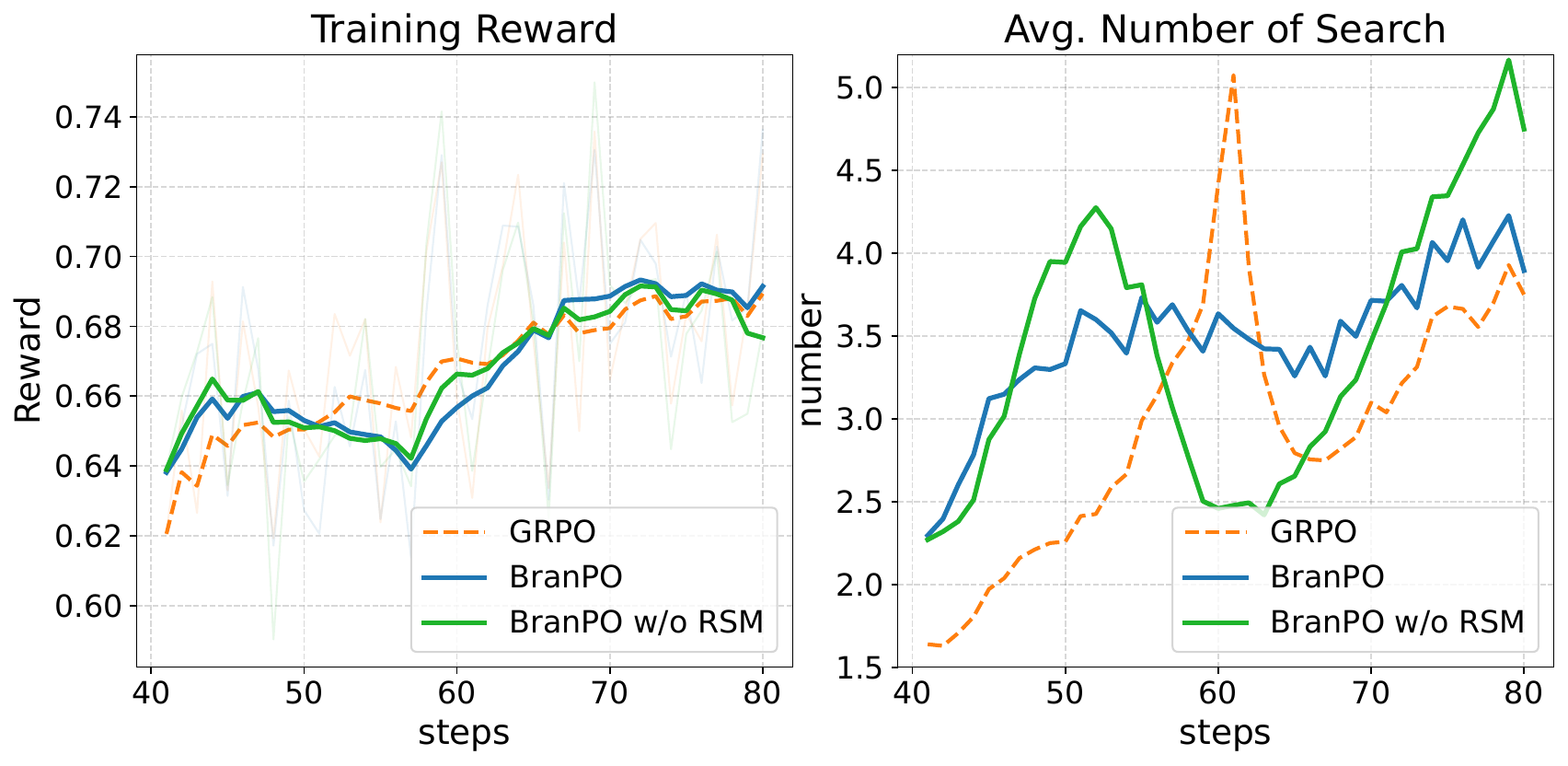}
    \caption{Training dynamics when scaling the number of training interactions to 8, showing the reward and the average number of search steps during training for GRPO, BranPO, and BranPO without RSM on Qwen3-4B-Instruct. RSM effectively suppresses unnecessary growth in search steps.}
    \label{fig:ablation_training}
\end{figure*}
\subsection{Suffix Updates as Preference-Style Optimization}
\label{app:dpo_like}
BranPO's suffix learning signal becomes preference-style because branches share the same prefix $\mathcal{B}$ and differ only in the continuation $\mathcal{C}$.

Suppose from a fixed prefix $\mathcal{B}$ we obtain a ``good'' suffix $\mathcal{C}^+$ and a ``bad'' suffix $\mathcal{C}^-$ such that
$$
R(\mathcal{B},\mathcal{C}^+) > R(\mathcal{B},\mathcal{C}^-)
$$
This defines an implicit preference pair conditioned on $\mathcal{B}$.

Consider the suffix-only log-probabilities $\log \pi_\theta(\mathcal{C}\mid \mathcal{B})$.
If the branch advantages are centered so that the two suffixes receive opposite weights,
$$
\hat{A}_{\mathrm{branch}}(\mathcal{C}^+) = +c,\qquad \hat{A}_{\mathrm{branch}}(\mathcal{C}^-) = -c
$$
then the suffix contribution to the policy-gradient term is proportional to
\begin{equation}
\begin{multlined}
-c\;\nabla_\theta \log \pi_\theta(\mathcal{C}^+\mid \mathcal{B})
+ c\;\nabla_\theta \log \pi_\theta(\mathcal{C}^-\mid \mathcal{B}) \\
= -c\;\nabla_\theta \Big(\log \pi_\theta(\mathcal{C}^+\mid \mathcal{B}) - \log \pi_\theta(\mathcal{C}^-\mid \mathcal{B})\Big)
\end{multlined}
\end{equation}
Thus, BranPO's suffix updates increase the conditional log-likelihood ratio
$$
\Delta_\theta(\mathcal{B})
\triangleq
\log \pi_\theta(\mathcal{C}^+\mid \mathcal{B}) - \log \pi_\theta(\mathcal{C}^-\mid \mathcal{B})
$$
while simultaneously applying the same KL regularization to $\pi_{\mathrm{ref}}$ as in GRPO.

\paragraph{Connection to DPO.}
DPO optimizes a logistic preference likelihood with a reference correction. For a conditional preference pair $(\mathcal{B},\mathcal{C}^+,\mathcal{C}^-)$, define

    \begin{align}
        z_\theta(\mathcal{B})
&\triangleq
\beta\Big(
\Delta_\theta(\mathcal{B})
-
\Delta_{\mathrm{ref}}(\mathcal{B})
\Big),
\\
\Delta_{\mathrm{ref}}(\mathcal{B})
&\triangleq
\log \pi_{\mathrm{ref}}(\mathcal{C}^+\mid \mathcal{B}) - \log \pi_{\mathrm{ref}}(\mathcal{C}^-\mid \mathcal{B})
    \end{align}

A DPO-style loss on suffixes is
$$
\mathcal{L}_{\mathrm{DPO\text{-}suffix}}(\theta)
=
-\mathbb{E}\Big[\log \sigma\big(z_\theta(\mathcal{B})\big)\Big]
$$
Its gradient for one pair is
\begin{equation}
    \nabla_\theta \mathcal{L}_{\mathrm{DPO\text{-}suffix}}(\theta)
=
-\beta\big(1-\sigma(z_\theta(\mathcal{B}))\big)\;
\nabla_\theta \Delta_\theta(\mathcal{B})
\end{equation}
When $z_\theta(\mathcal{B})$ is near $0$ (i.e., the model is near-indifferent between $\mathcal{C}^+$ and $\mathcal{C}^-$ under the reference-corrected logit),
$$
1-\sigma(z_\theta(\mathcal{B})) \approx \frac{1}{2}
$$
so the DPO gradient is approximately a constant multiple of $-\nabla_\theta \Delta_\theta(\mathcal{B})$, matching the BranPO suffix direction above (up to a scalar and the explicit KL penalty already present in GRPO/BranPO).
Therefore, BranPO's suffix optimization is \emph{DPO-like} in the precise sense of performing pairwise log-ratio improvement on contrastive continuations conditioned on the same intermediate state.

\paragraph{Takeaway.}
In search agents, errors concentrate near the tail.
Branching reuses the same expensive prefix $\mathcal{B}$ and explores only the suffix distribution $\pi_\theta(\cdot\mid\mathcal{B})$, yielding:
\begin{itemize}
\item \textbf{Compute concentration:} additional sampling cost scales with suffix length rather than full trajectory length.
\item \textbf{Denser learning signal:} contrastive continuations provide a strong discriminative signal localized to the tail, while prefix gradients are stabilized by $r_{\mathrm{base}}(\mathcal{B})$ (Appendix~\ref{app:variance}).
\end{itemize}
This explains why BranPO can be substantially more sample-efficient than GRPO when failures are dominated by late-stage synthesis.

\section{Additional Analysis of Redundant Step Masking}
\label{app:rsm_analysis}

This section provides additional analysis of the role of Redundant Step Masking (RSM) under extended training horizons. Starting from the first-stage trained models, we continue optimization into the second stage with the maximum number of training interactions scaled to 8.

As shown in Figure~\ref{fig:ablation_training}, RSM plays a critical role in stabilizing long-horizon optimization by suppressing the unnecessary growth of search steps. In contrast, removing RSM leads to less stable training dynamics, characterized by excessive exploration and a rapid increase in search depth. These results indicate that RSM is particularly important for maintaining efficient and controlled search behavior when scaling the interaction horizon.

\section{Case Study}
\subsection{Correctness Non-Monotonic}
\label{app:tail_errors}
We present representative cases illustrating \textbf{non-monotonic correctness} in multi-turn agentic trajectories. All trajectories shown in these tables are generated by Qwen2.5-7B-Instruct after SFT.

Table~\ref{tab:case_study:self_correction} presents a positive example of self-correction within the agentic trajectory. In this case, the initial search returns documents that primarily mention Juan Jufr\'e in the context of his activities in Chile. Based on this incomplete evidence, the agent initially misjudges the situation, suspecting that the question's reference to an ``Argentine governor'' may be a spelling error or a confusion with another historical figure. Crucially, however, the agent does not abandon the task or fabricate an answer. Instead, it recognizes the uncertainty in its intermediate conclusion and decides to refine its search keywords to further verify and correct its earlier judgment, eventually arriving at the correct answer. This case highlights the importance of persistent verification and adaptive search strategies in mitigating non-monotonic correctness.

Table~\ref{tab:case_study:insufficient} illustrates the first failure mode. In this case, the agent retrieves an outdated piece of information regarding the finance minister of Maharashtra but fails to issue follow-up queries to verify whether this information remains current. Instead of continuing the search to obtain the most up-to-date answer, the agent prematurely draws a conclusion based on the stale result, ultimately producing an incorrect answer.

Table~\ref{tab:case_study:hallucianation} demonstrates the second failure mode. Here, the agent successfully retrieves documents that collectively contain all the information needed to answer the question correctly. However, rather than carefully synthesizing and summarizing the retrieved evidence, the agent directly generates an answer without grounding it in the search results, leading to a hallucinated response that contradicts the very documents it had gathered.

\begin{table*}[hb]
    \centering
    \small 
    \begin{tabularx}{\textwidth}{
        >{\hsize=1\hsize}X  
        |                   
        >{\hsize=1\hsize}X  
    }
    \toprule
    
    \multicolumn{2}{p{\dimexpr\textwidth-2\tabcolsep}}{
        \textbf{Question:} Juan Jufr\'e was the governor of the Argentine province that historially comprised which provinces?
    } \\ 
    \addlinespace[0.2em]
    \multicolumn{2}{p{\dimexpr\textwidth-2\tabcolsep}}{
        \textbf{Ground Truth:} San Juan, San Luis and Mendoza
    } \\
    
    \midrule
    
    \multicolumn{2}{p{\dimexpr\textwidth-2\tabcolsep}}{
        \think{I need to determine which Argentine province Juan Jufr\'e was the governor of, and then find out which provinces were historically part of that province.}\newline
        \search{who was Juan Jufr\'e} \newline
        \info{[Document 1] Jofre Jofre or Jofr\'e may refer to the following people - Given name - Jofre de Foix\`a (died c. 1300), Italian poet, musician......
[Document 2] Juan Jufr\'e Accomplishments. Jufr\'e made a fortune developing important economic activities in Chile. He built the first shipyard next to the Maule River, in his encomienda of Pocoa.......
[Document 3] Juan Jufr\'e r a textile factory amid his two encomiendas (Mataquito and Peteroa) on the banks of the Mataquito River........
}\newline
\think{I found that \"Juan Jufr\'e\" is a name associated with Chile, not Argentina. However, the question refers to \"Juan Jufr\'e\" as the governor of an Argentine province. This suggests a possible misspelling or confusion. The name \"Juan Jufr\'e\" might be a Spanish variant of \"Juan Freire\" or similar. In Argentina, a common name is \"Juan Jos\'e\" or \"Juan\", and \"Jufr\'e\" might be a misspelling of \"Fr\'eda\" or \"F\'erte\". Alternatively, it could be a confusion with \"Juan Bautista\" or \"Juan de la Cruz\". Another possibility is that the name is a misspelling of \"Juan Bautista Alberdi\" or \"Juan Bautista\" related to a province. I will conduct a search to find the correct name of the Argentine governor possibly intended in the question.}\newline
\search{governor of Argentine province with name similar to Juan Jufr\'e}\newline
 \info{ [Document 1] Juan Jufr\'e de Loayza y Montesa was a Spanish conquistador who participated in the 1541 expedition of Pedro de Valdivia to Chile......

[Document 2] Governor of Jujuy. Province The Governor of Jujuy is a citizen of the Jujuy Province, in Argentina, holding the office of governor for the corresponding period.......

[Document 3] Governor of Formosa Province The Governor of Formosa is a citizen of the Formosa Province, in Argentina.......}
\newline
     \think{From the search, I found that Juan Jufr\'e is referred to as the governor of the Argentine province of Cuyo. Now I need to determine which provinces were historically part of Cuyo.}\newline
    \search{provinces that historically comprised Cuyo in Argentina}\newline
    \info{
[Document 1] Cuyo Province. The Province of Cuyo was a historical Province of Argentina. Created on 14 November 1813 by a decree issued by the Second Triumvirate, it had its capital on Mendoza, and was composed by the territories of the present Argentine Provinces of Mendoza, San Juan and San Luis.......
[Document 2] Cuyo (Argentina)  Cuyo is the name given to the wine-producing, mountainous area of central-west Argentina. Historically it comprised the provinces of San Juan, San Luis and Mendoza.......
[Document 3] For centuries in colonial times, the region was part of the Captaincy General of Chile, at the time part of the Viceroyalty of Peru.......
    }\newline
    \think{The Argentine province that Juan Jufr\'e was the governor of is Cuyo, which historically comprised the provinces of Mendoza, San Juan, and San Luis.}\newline
    \answer{Mendoza, San Juan, and San Luis}
    }
        \\
    \bottomrule
    \end{tabularx}
    \caption{Case Study: Self correction.}
 \label{tab:case_study:self_correction}
\end{table*}

\begin{table*}[t]
    \centering
    \small 

    \begin{tabularx}{\textwidth}{
        >{\hsize=1\hsize}X  
        |                   
        >{\hsize=1\hsize}X  
    }
    \toprule
    
    \multicolumn{2}{p{\dimexpr\textwidth-2\tabcolsep}}{
        \textbf{Question:} Who is the finance minister of the state that generates the most revenue in India?
    } \\ 
    \addlinespace[0.2em]
    \multicolumn{2}{p{\dimexpr\textwidth-2\tabcolsep}}{
        \textbf{Ground Truth:} Sudhir R. Mungantiwar
    } \\
    
    \midrule
    
    \multicolumn{2}{c}{\textbf{--- Shared Trajectory ---}} \\
    \multicolumn{2}{p{\dimexpr\textwidth-2\tabcolsep}}{
        \think{To answer this question, I first need to find out which state in India generates the most revenue. India is a federal structure, and each state has the authority to levy taxes, but not all states are equal in terms of revenue generation.}\newline
        \search{state with the most revenue generation in India} \newline
        \info{ [Document 1] Income in India the world's largest middle income market.

Compared to other countries......

[Document 2] Economic development in India and even reach the government's 2011 target of ......

[Document 3] Economy of India all crops has grown since 1950, due to the special emphasis placed on agriculture in the five-year plans and steady improvements in irrigation, technology......}\newline
\think{From the information provided, it seems that the economy of states like Gujarat, Maharashtra, Tamil Nadu, and others have been growing significantly. However, the specific state that generates the most revenue is not explicitly stated. Typically, the state of Maharashtra is known for being a significant contributor to the country's economy and revenue, due to its industrial and commercial activities. Let's verify this with specific data.}\newline
\search{finance minister of Maharashtra}\newline
 \info{ [Document 1] Diliprao Deshmukh Diliprao Deshmukh

Diliprao Dagdojirao Deshmukh () (born 18 April 1950) is an Indian politician, ......

[Document 2]Eknath Khadse Eknath Khadse

Eaknath Ganpatrao Khadse (born 2 September 1952) is a politician and senior leader of the Bharatiya Janata Party (BJP) in Maharashtra state. He is a Member of Legislative Assembly of Maharashtra from Muktainagar taluka......

[Document 3] Sanjay Shamrao Dhotre Sanjay Shamrao Dhotre

Sanjay Shamrao Dhotre (born 26 February 1959) is an Indian politician from ......}
    } \\
    
    \midrule
    
    \textbf{Case 1: Failure} & \textbf{Case 2: Success} \\
    \midrule
    
    \think{Given the information, the state of Maharashtra, which is known for being economically significant, has had a finance minister in recent years. Eknath Khadse and Sanjay Dhotre have both held the position of Finance Minister of Maharashtra. To definitively answer the question, we need the current finance minister. Since the date of the query is not specified, let's assume we are looking for the most current information.

Since Sanjay Dhotre is currently a minister in the central government, and without specific recent data, let's consider the most recent finance minister of Maharashtra, which is Eknath Khadse as of the last available information.}\newline
\answer{Eknath Khadse}
    
    & 
    \think{Based on the information provided, Diliprao Deshmukh was a Minister of State for Finance \& Planning in the state of Maharashtra. More notably, Eknath Khadse held the position of Finance Minister of Maharashtra from 2014 to 2015. Since Maharashtra is one of the states that generate significant revenue for India, we can conclude that the finance minister of the state that generates the most revenue in India, in this case, is likely one from Maharashtra. However, to be precise, we need to identify the current finance minister of Maharashtra.}\newline
    \search{current finance minister of Maharashtra}\newline
    \info{[Document 1] Eknath Khadse Eknath Khadse

Eaknath Ganpatrao Khadse (born 2 September 1952) is a politician and senior leader of the Bharatiya Janata Party (BJP) ......

[Document 2]Sanjay Shamrao Dhotre Sanjay Shamrao Dhotre

Sanjay Shamrao Dhotre (born 26 February 1959) is an Indian politician from the state of Maharashtra ......

[Document 3]Sudhir Mungantiwar Sudhir Mungantiwar

Sudhir Mungantiwar (born 30 July 1962) is a Legislator from the state of Maharashtra, India. He currently serves as the Cabinet Minister......}\newline
\think{Based on the current information:
\begin{itemize}
    \item Eknath Khadse was Finance Minister of Maharashtra from 2014 to 2015.
    \item Sanjay Shamrao Dhotre is a Minister of State in the Government of India, not currently the Finance Minister of Maharashtra.
    \item Sudhir Mungantiwar is the current Cabinet Minister of Finance \& Planning in the Government of Maharashtra, having taken office in October 2014.
\end{itemize}
}\newline
\answer{Sudhir Mungantiwar}

    \\
    \bottomrule
    \end{tabularx}
    \caption{Case Study: Insufficient Search.}
    \label{tab:case_study:insufficient}
\end{table*}

\begin{table*}[t]
    \centering
    \small 
    \begin{tabularx}{\textwidth}{
        >{\hsize=1\hsize}X  
        |                   
        >{\hsize=1\hsize}X  
    }
    \toprule
    
    \multicolumn{2}{p{\dimexpr\textwidth-2\tabcolsep}}{
        \textbf{Question:} Which of these artists—Artemisia Gentileschi, Rosa Barba, or Lavinia Fontana—has the earliest birth date?
    } \\ 
    \addlinespace[0.2em]
    \multicolumn{2}{p{\dimexpr\textwidth-2\tabcolsep}}{
        \textbf{Ground Truth:} Lavinia Fontana
    } \\
    
    \midrule
    
    \multicolumn{2}{c}{\textbf{--- Shared Trajectory ---}} \\
    \multicolumn{2}{p{\dimexpr\textwidth-2\tabcolsep}}{
        \think{I need to find the birth dates of each artist to compare and determine which one is the earliest.}\newline
        \search{Artemisia Gentileschi birth date} \newline
        \info{ [Document 1] Gentileschi Gentileschi

Gentileschi can refer to:

Orazio Gentileschi (1563-1639), an Italian painter

Artemisia Gentileschi (1593-1653), a Baroque painter (and Orazio's daughter)

[Document 2] Artemisia Gentileschi and expressive painters of her generation, and is due to be celebrated at the exhibition at the National Gallery in London in 2020.......

[Document 3] Artemisia (film) Artemisia (film)

Artemisia is a 1997 French-German-Italian biographical film about Artemisia Gentileschi, the female Italian Baroque painter. The film was directed by Agnès Merlet, and stars Valentina Cervi and ......}\newline
\think{From the search, I found that Artemisia Gentileschi was born on July 8, 1593. Now I will find the birth dates of Rosa Barba and Lavinia Fontana.}\newline
\search{Rosa Barba birth date}\newline
 \info{ [Document 1] Rosa Barba Rosa Barba

Rosa Barba (born 1972, Agrigento, Italy) is a German-Italian visual artist and filmmaker. Barba is known for using the medium of film and its materiality to create cinematic film installations......

[Document 2] Meche Barba Meche Barba

Meche Barba (born Mercedes Barba Feito; September 24, 1922 – January 14, 2000), was an American-born Mexican film actress and dancer of the Golden age of Mexican cinema in the 1940s and 1950s.......

[Document 3] Rosalba of the 2012 Italian Olympic Judo team.

Rosalba Neri (b. 1939), an Italian actress also known by the name "Sara Bey" or "Bay".......}
\newline
 \think{From the search, Rosa Barba was born in 1972. Now I will find the birth date of Lavinia Fontana.}\newline
    \search{Lavinia Fontana birth date}\newline
    \info{
    [Document 1] Lavinia Fontana Lavinia Fontana

Lavinia Fontana (1552 – August 11, 1614) was a Bolognese Mannerist painter best known for her portraiture. She was trained by her......

[Document 2] Lavinia Fontana mythological themes which sometimes included female nudes.

Fontana married Gian Paolo Zappi (alternate spellings include......

[Document 3] Lavinia Fontana was elected into the Accademia di San Luca of Rome. She died in the city on August 11, 1614 and was subsequently buried at Santa Maria sopra Minerva......
    }
    } 
   
    \\
    
    \midrule
    
    \textbf{Case 1: Failure} & \textbf{Case 2: Success} \\
    \midrule
    
    \think{Based on the information provided, Lavinia Fontana was born in 1552.}\newline
\answer{Artemisia Gentileschi}
    
    & 
    \think{From the search, Lavinia Fontana was born in 1552. Now that we have the birth years of all three artists:
\begin{itemize}
    \item Artemisia Gentileschi: 1593
    \item Rosa Barba: 1972
    \item Lavinia Fontana: 1552
\end{itemize}
Lavinia Fontana has the earliest birth date. }\newline
    
\answer{Lavinia Fontana}

    \\
    \bottomrule
    \end{tabularx}
    \caption{Case Study: Hallucination.}
    \label{tab:case_study:hallucianation}
\end{table*}

\subsection{Redundant Steps}
\label{app:redundant}
Here we show an example of redundant steps generated by Qwen2.5-7B-Instruct. As shown in Table~\ref{tab:case_study:redundant}, the model is able to produce the answer confidently after the first search, however, it does another two search actions to get the answer.

\subsection{BranPO}
We present example trajectories generated by two BranPO-trained models. As shown in Table~\ref{tab:case_study:success}, the models are able to accurately decompose the questions into multiple search steps and extract key information from the search engine responses. However, we observe that the generated queries still contain meaningless tokens such as “who is.” Table~\ref{tab:case_study:fail} shows a failure case. Toward the end of this trajectory, the model’s outputs become unstable. This further indicates that finer-grained rewards or higher-quality cold-start initialization are necessary to improve training stability.

\begin{table*}[t]
    \centering
    \small 
    \begin{tabularx}{\textwidth}{
        >{\hsize=1\hsize}X  
        |                   
        >{\hsize=1\hsize}X  
    }
    \toprule
    
    \multicolumn{2}{p{\dimexpr\textwidth-2\tabcolsep}}{
        \textbf{Question:} Who is the author of The Confidential Clerk?
    } \\ 
    \addlinespace[0.2em]
    \multicolumn{2}{p{\dimexpr\textwidth-2\tabcolsep}}{
        \textbf{Ground Truth:} T. S. Eliot
    } \\
    
    \midrule
    
    \multicolumn{2}{p{\dimexpr\textwidth-2\tabcolsep}}{
        \think{To answer who is the author of "The Confidential Clerk," I need to search for this title to see if I have any information about it.}\newline
        \search{The Confidential Clerk book} \newline
        \info{[Document 1] The Confidential Clerk The Confidential Clerk
The Confidential Clerk is a comic verse play by T. S. Eliot.
Synopsis.
Sir Claude Mulhammer......
[Document 2] The Confidential Clerk in 1953, before transferring to the West End. It played for almost a year at the Duke of York's Theatre before embarking on a UK Tour. It was produced by Henry Sherek and directed by E......
[Document 3] U.S.A. Confidential U.S.A. Confidential
U.S.A. Confidential is a 1952 book written by Jack Lait and Lee Mortimer (Crown Publishers). Its theme is crime and corruption.......
}\newline
\think{Now I have a basic idea of what "The Confidential Clerk" is. It's a comic verse play. I need to find out who wrote it.}\newline
\search{who wrote the confidential clerk play by t.s. eliot}\newline
 \info{ [Document 1] Rosa Barba Rosa Barba

Rosa Barba (born 1972, Agrigento, Italy) is a German-Italian visual artist and filmmaker. Barba is known for using the medium of film and its materiality to create cinematic film installations......

[Document 2] Meche Barba Meche Barba

Meche Barba (born Mercedes Barba Feito; September 24, 1922 – January 14, 2000), was an American-born Mexican film actress and dancer of the Golden age of Mexican cinema in the 1940s and 1950s.......

[Document 3] Rosalba of the 2012 Italian Olympic Judo team.

Rosalba Neri (b. 1939), an Italian actress also known by the name "Sara Bey" or "Bay".......}
\newline
     \think{Now we have confirmed that "The Confidential Clerk" is a play written by T. S. Eliot. I have the information I need to answer the question.}\newline
    \search{who is the author of The Confidential Clerk}\newline
    \info{
[Document 1] The Confidential Clerk The Confidential Clerk
The Confidential Clerk is a comic verse play by T. S. Eliot.
Synopsis.
Sir Claude Mulhammer......
[Document 2] The Confidential Clerk in 1953, before transferring to the West End. It played for almost a year at the Duke of York's Theatre before embarking on a UK Tour.......
[Document 3] U.S.A. Confidential U.S.A. Confidential
U.S.A. Confidential is a 1952 book written by Jack Lait and Lee Mortimer (Crown Publishers). Its theme is crime and corruption......
    }\newline
    \think{Based on the provided information, "The Confidential Clerk" is a play written by T. S. Eliot. Therefore, the author of "The Confidential Clerk" is T. S. Eliot.}\newline
    \answer{T. S. Eliot}
    } 
        \\
    \bottomrule
    \end{tabularx}
    \caption{Case Study: Redundant Step Generated by Qwen2.5-7B-Instruct from PopQA.}
 \label{tab:case_study:redundant}
\end{table*}

\begin{table*}[t]
    \centering
    \small 
    \begin{tabularx}{\textwidth}{
        >{\hsize=1\hsize}X  
        |                   
        >{\hsize=1\hsize}X  
    }
    \toprule
    
    \multicolumn{2}{p{\dimexpr\textwidth-2\tabcolsep}}{
        \textbf{Question:} Who is the federal leader of the political party Ken Epp belongs to?
    } \\ 
    \addlinespace[0.2em]
    \multicolumn{2}{p{\dimexpr\textwidth-2\tabcolsep}}{
        \textbf{Ground Truth:} Andrew Scheer
    } \\
    
    \midrule
    
    \multicolumn{2}{p{\dimexpr\textwidth-2\tabcolsep}}{
        \think{I need to determine the name of the political party that Ken Epp belongs to, and then find out the federal leader of that party.}\newline
        \search{who is ken epp} \newline
        \info{[Document 1] Ken Epp Ken Epp

Ken Epp (born May 11, 1939 in Swift Current, Saskatchewan) is a Canadian politician.

Epp was a member of the Conservative Party of Canada in the House of Commons of Canada, representing the riding of......
[Document 2] Ken Epp in Swift Current, Saskatchewan) is a Canadian politician.

Epp was a member of the Conservative Party of Canada in......
[Document 3] Epp (surname) ian and former politician

Herb Epp (1934–2013), Canadian politician

Jake Epp (born 1939), Canadian executive and former politician

Ken Epp (born 1939), Canadian politician......
}\newline
\think{Ken Epp is a Canadian politician. He is a member of the Conservative Party of Canada. I need to find out the current federal leader of the Conservative Party of Canada.}\newline
\search{who is the federal leader of the Conservative Party of Canada}\newline
 \info{ [Document 1] 2017 Conservative Party of Canada leadership election 2017 Conservative Party of Canada leadership election......

[Document 2] Conservative Party of Canada Conservative Party of Canada

The Conservative Party of Canada (), colloquially known as the Tories, is a federal political party in Canada.......

[Document 3] Conservative Party of Canada transfer of federal government powers to the provinces modeled after the Meech Lake Accord and a tougher stand on "law and order" issues........}\newline
     \think{Ken Epp is a Canadian politician and a member of the Conservative Party of Canada. The Conservative Party of Canada is the political party he belongs to. The current federal leader of the Conservative Party of Canada is Andrew Scheer. Therefore, the federal leader of the political party Ken Epp belongs to is Andrew Scheer.}\newline
    \answer{Andrew Scheer}}
    \\
    \bottomrule
    
    \end{tabularx}
    \caption{Case Study: Successful search trajectory generated by Qwen3-4B-Instruct-BranPO from MuSiQue.}
    \label{tab:case_study:success}
\end{table*}

\begin{table*}[t]
    \centering
    \small 
    \begin{tabularx}{\textwidth}{
        >{\hsize=1\hsize}X  
        |                   
        >{\hsize=1\hsize}X  
    }
    \toprule
    
    \multicolumn{2}{p{\dimexpr\textwidth-2\tabcolsep}}{
        \textbf{Question:} Who is Hieronymus, Son Of Charles Martel's paternal grandmother?
    } \\ 
    \addlinespace[0.2em]
    \multicolumn{2}{p{\dimexpr\textwidth-2\tabcolsep}}{
        \textbf{Ground Truth:} Alpaida
    } \\
    
    \midrule
    
    \multicolumn{2}{p{\dimexpr\textwidth-2\tabcolsep}}{
        \think{To answer this, I need to know who Hieronymus is and who is his father. I also need to find information about the paternal grandmother of Hieronymus. However, I'm not sure if Hieronymus is a person I'm familiar with, especially in relation to Charles Martel. I will start with a general search to see if I can find any information about Hieronymus.} \newline
        \search{Hieronymus Son Of Charles Martel}\newline
        \info{[Document 1] Hieronymus, son of Charles Martel Hieronymus, son of Charles Martel

Hieronymus (c. 722 - after 782), was the son of Charles Martel and his mistress Ruodhaid, and so was the brother of Bernard......
[Document 2] Carolingians descended from Charles Martel Carolingians descended from Charles Martel

This is a partial list of male descendants from Charles Martel (686–741) for fifteen generations.......
[Document 3] Hieronymus, son of Charles Martel rad (d. 31 January 826), Abbé de Saint-Quentin and an imperial ''missi of Charlemagne in 806.

Richarda, married Nithard.....
}\newline
\think{Now I have a basic information about Hieronymus, son of Charles Martel. He is a son of Charles Martel and Ruodhaid. I need to find out who Hieronymus' paternal grandmother is. Given that Hieronymus is a son of Charles Martel, his paternal grandmother would be the mother of Charles Martel. I need to find who Charles Martel's mother is.}\newline
\search{paternal grandmother of charles martel}\newline
 \info{[Document 1] Martel Martel

Martel may refer to:

People.......

[Document 2]Carolingians descended from Charles Martel Carolingians descended from Charles Martel

This is a partial list of male descendants from Charles Martel (686–741) for fifteen generations......
[Document 3] Charles Martel (disambiguation) Charles Martel (disambiguation)

Charles Martel (c. 688–741) was a Frankish military and political leader.

Charles Martel may also refer to:........}\newline
     \think{There is no specific information about the mother of Charles Martel in these documents. Given that Hieronymus is the son of Charles Martel and Ruodhaid, and we need to find Hieronymus' paternal grandmother, we already know that the paternal grandmother of Hieronymus is Ruodhaid. Now, we need to find who Hieronymus' paternal grandmother's father is, which is the same as finding who is the father of Ruodhaid. However, we don't have any information about Ruodhaid's father.

Given the available information, we can conclude that Hieronymus' paternal grandmother is Ruodhaid. The question asks for Hieronymus, Son Of Charles Martel's paternal grandmother. Given that Hieronymus is the son of Charles Martel and Ruodhaid, and we need to find the paternal grandmother of Hieronymus, the answer is Ruodhaid.}\newline
    \answer{Ruodhaid}}
    \\
    \bottomrule
    
    \end{tabularx}
    \caption{Case Study: Failed search trajectory generated by Qwen2.5-7B-Instruct-BranPO from 2WikiMultihopQA.}
    \label{tab:case_study:fail}
\end{table*}

\end{document}